%% file: main.tex
\documentclass[twoside,11pt]{article}

\usepackage{blindtext}

\usepackage[preprint]{jmlr2e}

\usepackage{xcolor}
\usepackage{nicefrac}
\usepackage{microtype}
\usepackage{bbm,bm}
\usepackage{enumitem}
\usepackage{amsmath,amssymb,mathtools,amsfonts,amsthm}
\usepackage{thmtools,thm-restate}
\usepackage[capitalise,noabbrev]{cleveref}
\usepackage{algorithm}
\usepackage{algorithmicx}
\usepackage{algpseudocode}
\usepackage{parskip}
\usepackage{caption}
\usepackage{booktabs}

\input{macros.tex}

\newcommand{\blue}[1]{\textcolor{blue}{#1}}

\usepackage[colorinlistoftodos,bordercolor=orange,backgroundcolor=orange!20,linecolor=orange]{todonotes}

\hypersetup{
    colorlinks,
    linkcolor={blue!50!black},
    citecolor={blue!50!black},
    urlcolor={blue!80!black}
}

\declaretheorem[name=Theorem]{thm}
\crefname{thm}{Theorem}{Theorems}
\declaretheorem[name=Lemma]{lem}
\crefname{lem}{Lemma}{Lemmas}
\declaretheorem[name=Proposition]{prop}
\crefname{prop}{Proposition}{Propositions}
\declaretheorem[name=Corollary]{cor}
\crefname{cor}{Corollary}{Corollaries}

\crefname{fact}{Fact}{Facts}

\crefname{defn}{Definition}{Definitions}

\crefname{assmpt}{Assumption}{Assumptions}

\crefname{remark}{Remark}{Remarks}

\crefname{claim}{Claim}{Claims}

\usepackage{lastpage}
\jmlrheading{23}{2022}{1-\pageref{LastPage}}{1/21; Revised 5/22}{9/22}{21-0000}{Author One and Author Two}

\ShortHeadings{}{}
\firstpageno{1}

\begin{document}

\title{Gradient Descent on Logistic Regression: \\ Do Large Step-Sizes Work with Data on the Sphere?}

\author{%
    \name Si Yi Meng \email sm2833@cornell.edu \\
    \addr Department of Computer Science\\
    Cornell University\\
    Ithaca, NY, USA
    \AND
    \name Baptiste Goujaud \email baptiste.goujaud@telecom-sudparis.eu \\ 
    \addr SAMOVAR, Télécom SudParis \\
    Institut Polytechnique de Paris \\
    Palaiseau, France
    \AND
    \name Antonio Orvieto \email antonio@tue.ellis.eu \\
    \addr ELLIS Institute Tübingen\\
    MPI for Intelligent Systems\\
    Tübingen AI Center, Germany
    \AND
    \name Christopher De Sa \email cdesa@cs.cornell.edu \\
    \addr Department of Computer Science\\
    Cornell University\\
    Ithaca, NY, USA
}

\editor{My editor}

\maketitle
\addtocontents{toc}{\protect\setcounter{tocdepth}{-1}}

\begin{abstract}
    Gradient descent on logistic regression 
    has many fascinating properties not fully captured by 
    classical analyses.
    For linearly separable data, 
    the iterates converge in direction to the maximum-margin 
    separator regardless of how large the step size is. 
    For non-separable data, however, cycles may occur 
    even for step sizes below the local stability threshold of $2/\lambda^*$, 
    where $\lambda^*$ is the largest eigenvalue 
    of the Hessian at the solution. 
    We ask whether requiring all examples to have 
    equal norm is a sufficient condition for global 
    convergence below this threshold. We prove that this is 
    indeed true in one dimension. In higher dimensions,
    even equal-norm datasets can produce cycles, although 
    global convergence holds for the population loss  
    when the examples are uniformly distributed on the sphere.
    Finally, in one dimension, we 
    relax the equal-norm assumption to an annulus condition
    that ensures global convergence below the stability threshold,
    suggesting a path towards analogous sufficient conditions 
    in higher dimensions.
\end{abstract}

\input{sections/intro.tex}
\input{sections/literature.tex}

\input{sections/1d-convergence.tex}
\input{sections/high-d-counterexample.tex}
\input{sections/high-d-convergence.tex}
\input{sections/1d-suff-cond.tex}

\input{sections/conclusion.tex}

\vskip 0.2in

\acks{We would like to thank Frederik Kunstner
for helpful discussions and feedback on the manuscript.
A.O. acknowledges the financial support of the Hector Foundation.
C.D. was supported by the NSF CAREER Award (2046760). 
S.M. is supported by the European Union (ERC grant CASPER 101162889). 
The French government also partly funded this work under the management 
of Agence Nationale de la Recherche as part of the ``France 2030'' 
program, reference ANR-23-IACL-0008 (PR[AI]RIE-PSAI). 
Views and opinions expressed are however those of the authors only. 
This work was partially done during S.M.'s visit at the 
ELLIS Institute Tübingen.}

\aiuse{An earlier version of this work (arXiv \texttt{2507.11228v1}) was 
conducted and written entirely without AI assistance, while 
in the current version, we used AI tools to 
formulate mathematical claims, provide critical ingredients 
for proving mathematical claims, and assist in the writing of proofs.
In particular, we 
\begin{enumerate}
    \item hypothesized the existence of a low-dimensional 
    spherical counterexample and subsequently used AI tools to find the one 
    in \cref{sec:two-dim-counterexample}. We numerically verified 
    the cycle is stable and used AI to formulate its proof and write 
    the CAS certificates.
    \item conjectured the global convergence on the population objective 
    in arbitrary dimensions \cref{sec:population-obj}
    and used AI tools to formulate the proof.
    \item conjectured the sufficiency of the 
    one-dimensional annulus condition and 
    its proof is written with AI assistance.
\end{enumerate}
All AI-assisted proofs were checked line by line and significantly 
rewritten by the authors to improve interpretability. 
We take responsibility for the final content of this work, including text, 
claims, or artifacts produced with the aid of generative AI.
}
\vskip 0.2in
\bibliography{refs}

\appendix
\input{sections/appendix/appendix.tex}

\end{document}

%% file: macros.tex
\usepackage{xparse}

\def\eqref#1{equation~\ref{#1}}

\def\1{\bm{1}}

\def\zerovec{\mathbf{0}}

\DeclareMathAlphabet{\mathsfit}{\encodingdefault}{\sfdefault}{m}{sl}
\SetMathAlphabet{\mathsfit}{bold}{\encodingdefault}{\sfdefault}{bx}{n}

\def\gL{{\mathcal{L}}}

\def\gN{{\mathcal{N}}}

\newcommand{\E}{\mathbb{E}}

\newcommand{\R}{\mathbb{R}}

\DeclareMathOperator{\sign}{sgn}

\newcommand{\norm}[1]{\left\Vert #1\right\Vert}
\newcommand{\normsq}[1]{\left\Vert#1\right\Vert^{2}}

\newcommand{\abs}[1]{\left \vert#1\right \vert}
\newcommand{\paren}[1]{\left(#1\right)}
\newcommand{\brackets}[1]{\left[#1\right]}
\newcommand{\braces}[1]{\left\{#1\right\}}

\newcommand{\transpose}{^\mathsf{\scriptscriptstyle T}}
\newcommand{\inverse}{^{\scriptscriptstyle -1}}

\DeclareMathOperator{\Diag}{Diag}		

\DeclareMathOperator{\sech}{sech}

\DeclareMathOperator{\arctanh}{arctanh}

\def\loss{\gL}

\def\sphere{\mathbb{S}}

\def\Eweighted{\mathbb{E}_{(r)}}

%% file: sections/intro.tex
\section{Introduction}

Logistic regression is a fundamental problem in 
statistics and machine learning. In the age of deep learning and large 
language models, logistic regression remains valuable for its 
simplicity, interpretability, and efficiency, especially for 
tabular datasets with binary outcomes. Given a dataset consisting of 
$n$ examples $x_i\in\R^d$, each assigned a binary label $y_i=1$ or $-1$, 
we wish to find a linear separator $w^*\in\R^d$ that minimizes
\begin{align}
	\loss(w) \coloneqq \frac{1}{n}\sum_{i=1}^n \log \paren{ 1 + \exp( - y_i w\transpose x_i ) }.
\end{align}
As logistic regression typically 
does not have closed-form solutions, one often resorts to first-order 
methods like Gradient Descent (GD), which iteratively updates
\begin{align}
	w_{t+1} = T(w_t), \qquad  T(w) \coloneqq w - \eta \nabla \loss(w)
\end{align}
where $\eta>0$ is the step size. 
This simple update provides a clean setting for understanding 
how large step-sizes affect optimization dynamics.
Classical convex optimization theory tells us that if we use $\eta < 2/L$
where $L$ is the smoothness constant (global upper bound on the Hessian),
the function value is guaranteed to decrease monotonically 
and GD converges to $w^*$ (see, e.g., \citep{nesterov2018lectures}).
For the specific case of logistic regression, 
the step sizes under which global convergence 
is guaranteed can sometimes be much larger, allowing faster convergence 
\citep{wu2023implicit,wu2024large}.

When the dataset is linearly separable, there are generally 
infinitely many separating hyperplanes, and the objective 
approaches its infimum as the classifier norm tends to infinity.
In this case, the directions of $\braces{w_t}_t$ converge to the 
maximum-margin separator, a property known as the implicit bias 
\citep{soudry2018implicit}.
Remarkably, this convergence property persists under
\emph{arbitrarily} large step-sizes \citep{wu2023implicit}
which, despite potentially causing initial oscillations, yield faster 
asymptotic convergence \citep{wu2024large}.

The situation changes for non-separable data, where the minimizer 
is unique and finite.
It shouldn't come as a surprise 
that we can no longer choose arbitrarily-large step-sizes:
local convergence necessitates $\eta<2/\lambda^*$, where $\lambda^*$ 
is the largest eigenvalue of the Hessian at $w^*$.
If we view GD as a discrete-time nonlinear map 
then $2/\lambda^*$ corresponds to the 
local stability threshold. Above this threshold, GD will 
diverge, but below it, convergence is still possible
when initialized sufficiently close to the minimizer.
The favorable behaviour in the separable setting raises
a natural question: for non-separable data, does every step size 
below the local stability threshold also work globally?
That is, if we use step sizes of the form 
\begin{align}
	\label{eq:step-size}
	\eta = \gamma/\lambda^* \qquad \text{where} \qquad 
	\lambda^* \coloneqq \lambda_{\max}(\nabla^2\loss(w^*)),
\end{align}
can GD converge to $w^*$ globally using any $\gamma<2$?
As it turns out, the answer is no in general.
\citet{meng2024gradient} showed that, even when $d=1$, 
global convergence is guaranteed only for $\gamma\in(0,1)$, 
while for every $\gamma\in(1,2)$ there exist datasets 
and initializations for which GD converges to a cycle.
In higher dimensions, cycles exist for every $\gamma<2$. 
This local-to-global stability gap is 
perhaps unsurprising, rendering the favorable separable 
setting a special case. What's particularly interesting 
is the construction of the counterexamples.

Their proof technique involves constructing pathological datasets 
in the following way: for a fixed $\gamma$,
start with a base dataset consisting of only 
examples on the unit sphere such that GD with $\eta$ of the form 
\labelcref{eq:step-size} on this dataset follows a controlled 
trajectory. To this dataset, a few examples 
of large magnitude with specifically-chosen directions are then added.
These very large examples function as outliers, contributing
to the gradient only when the iterates land 
in certain regions, thereby kicking the trajectory into a cycle.
However, such counterexamples are unrepresentative 
of naturally occurring datasets. Could this problematic behaviour 
be avoided if these outlier mechanisms are ruled out?
If so, it means that under such data regularity conditions, 
GD on non-separable logistic regression can also converge globally for 
very large step-sizes, potentially much larger than the pessimistic 
$2/L$ bound all the way up to the $2/\lambda^*$ stability threshold.

This leads to our formal question:
\emph{is global convergence guaranteed for all $\gamma<2$ 
when examples are restricted to the unit sphere, 
i.e., $\norm{x_i}=1$}? A mere uniform scaling such as $\norm{x_i}\leq 1$
would not be meaningful: when $\eta=\gamma/\lambda^*$,
a simple argument (\cref{prop:gd-preserved-under-scaling}) 
shows that any GD trajectory on a dataset
is preserved under uniform scaling for the same $\gamma$.
We therefore consider the stronger condition $\norm{x_i}=1$,
which lets us ask whether, after eliminating magnitude variations,
the directions and multiplicities of the examples can still produce 
nonconvergent dynamics. Although exact normalization is far from realistic 
on datasets typically used with logistic regression, 
it does nonetheless arise naturally in several scenarios.
In text retrieval and classification, tf-idf vectors representing 
documents are usually normalized to account for document length before 
downstream classification tasks such as logistic regression 
\citep{salton1988term,manning2008introduction}. In facial recognition,
NormFace \citep{wang2017normface} and ArcFace \citep{deng2019arcface} 
similarly feed normalized embeddings to the final classification layer.
Contrastive learning has also been shown to benefit
from normalized representations \citep{chen2020simple}.
Although the latter settings are not identical to logistic regression, 
they motivate the spherical assumption both as an idealized model 
of commonly used feature representations, and as a stress test of whether 
exact normalization can prevent nonconvergent dynamics under large step-sizes.

To summarize our results, we first show that in one dimension, 
the spherical assumption indeed guarantees global convergence for 
all $\gamma\in(0,2)$.
In particular, convergence is oscillatory in the iterates 
when $\gamma\in(1,2)$. In higher dimensions, however, we show 
counterexamples exist both in high dimensions and in $d=2$.
These counterexamples share a common characteristic: 
despite the examples all having the same magnitude, their distribution 
is actually extremely skewed.
We then show that for features uniformly distributed on the sphere 
with arbitrary conditional label distributions, 
GD on the population objective converges globally for all $\gamma<2$. 
Finally, we relax the spherical assumption 
in one dimension to an annulus condition, 
suggesting a path towards analogous conditions in higher dimensions.

%% file: sections/literature.tex
\section{Related work}
\label{sec:related-work}

The behaviour of GD on logistic regression under large step-sizes 
has also been studied from different perspectives.
\citet{crawshaw2026tight} improved the $O(1/T^2)$ rate of \citet{wu2024large} 
in the separable but low-dimensional setting via a tight analysis of 
the oscillatory transition from the unstable to the stable regime. 
Several other works also explored the favorable large-step behaviour,
using increasing or adaptive step sizes to yield linear 
or minimax optimal rates \citep{axiotis2023gradient,kale2026stretched,zhang2026minimax}.
Furthermore, in the infinite-step-size limit, GD approaches 
a batch version of the perceptron algorithm \citep{tyurin2025logistic} 
while \citet{bao2026any} extended the arbitrary-step-size 
convergence property beyond logistic regression to broader classes of 
Fenchel-Young losses.
Adding an $\ell_2$ regularization to a separable logistic regression 
objective gives the problem a finite minimizer. In this setting, 
\citet{wu2026large} established similar convergence guarantees for some, 
but not all, step sizes below the local stability threshold. 
For unregularized non-separable logistic regression, 
\citet{liu2023implicit} found a stable $2$-cycle
on a symmetric two-example instance, but only exactly at or above
the local stability threshold. By contrast,
\citet{meng2024gradient} showed that a finite minimizer can
remain locally stable while coexisting with a stable cycle,
so that the threshold $2/\lambda^*$ is not in general sufficient
for global convergence.
The present work builds directly on this result
by asking whether restrictions on the geometry of the data, 
particularly equal example norms, can rule out such nonconvergent 
dynamics below the local stability threshold.

In neural network training, large step-size GD has also been 
studied in the Edge-of-Stability regime. This is a setting where 
along the optimization trajectory, the largest Hessian eigenvalue 
hovers right at or above $2/\eta$, 
while the loss can continue to decrease in the long run 
\citep{cohen2021gradient,damian2023self,song2023trajectory}. 
Under step sizes in this regime, GD has been shown to exhibit
periodic, chaotic, and fractal dynamics in simplified nonlinear 
and factored models 
\citep{chen2023beyond,chen2024from,liang2026gradient,hofmann2026map}. 
Our focus is instead on the entire range of step sizes \emph{below} 
the minimizer's local stability threshold, where the minimizer 
remains locally attracting but need not be globally attracting.

Our study is more closely related to the absolute stability of 
discrete-time dynamical systems. In particular, 
GD on logistic regression can be written as a Lur'e system---a linear system 
interconnected with a time-invariant 
nonlinearity $\varphi$ \citep{lur1944theory}.
Absolute stability means that the fixed point 
is globally asymptotically stable (GAS) 
for every nonlinearity in a given class. Such classes often 
require the nonlinearity to be sector- or slope-restricted, 
which, roughly speaking, means its input-output ratio 
or all of its secant slopes are respectively bounded.
Classical criteria such as the 
Tsypkin or Jury-Lee criteria combine restrictions on 
the nonlinearity with conditions on the linear subsystem to guarantee GAS 
\citep{tsypkin1964frequency,jury1964stability}.
Subsequent works proposed less conservative conditions 
based on Lyapunov functions and linear matrix inequalities (LMI)
\citep{gonzaga2012stability,park2015improved,kim2015stability,jungers2024discrete}. 
Since every linear interconnection $\varphi(z) = kz$ 
with $k$ bounded in some interval is itself an admissible nonlinearity, GAS 
necessitates stability of all the corresponding linear feedback systems. 
\citet{kalman1957physical} conjectured that this necessary condition is also 
sufficient, i.e., boundedness of $\varphi'$ 
implies GAS. The discrete-time version of this conjecture is true 
only in one dimension, while counterexamples have been constructed 
starting as low as $d=2$ \citep{heath2015second,seiler2020construction}. 
The local-to-global stability gap identified in \citet{meng2024gradient}
for GD on logistic regression closely resembles that of the Kalman conjecture.
In optimization, \citet{lessard2016analysis} developed a framework 
for analyzing first-order methods as discrete-time Lur'e systems 
using integral quadratic constraints that are LMI certificates for 
convergence. However, it is unclear how to apply their framework 
to step sizes computed relative to the local stability threshold 
of each problem instance. More generally, existing Lur'e stability criteria 
are only sufficient and may be too conservative 
for the structured dynamics of GD on logistic regression. 
Conversely, known counterexamples to the Kalman conjecture 
allow the Lur'e system to be chosen much more freely, 
and their constructions need not correspond to GD on a 
binary logistic regression problem. 
These limitations motivate a specialized analysis of 
GD on logistic regression under data constraints.

%% file: sections/1d-convergence.tex
\section{Global convergence in one dimension}
\label{sec:1d-convergence}

Without additional data restrictions, \citet[Theorem 1]{meng2024gradient} 
established global convergence in $d=1$ for all $\gamma\in(0,1)$. 
Under the spherical assumption, that is, $x_i=\pm 1$, 
we extend this guarantee to all $\gamma\in(0,2)$,
the full range of step sizes below the stability threshold.

\begin{thm}
	\label{thm:1d-convergence-sphere}
	Let $d=1$ and $\norm{x_i}=1$ for all $i$.
	Then for any $w_0\in\R$, GD iterates with step size $\eta$ of the form 
	\labelcref{eq:step-size} converges to $w^*$ for all $\gamma\in(0,2)$.
\end{thm}
\begin{proof}
	We only need to consider $\gamma \in (1,2)$, as the $\gamma \leq 1$ case 
	is already covered. We can absorb the labels into the features $x_i$
	and assume $y_i=1$ for all $i$. 
	Suppose there are $m$ copies of $x_i=1$ and 
	$n$ copies of $x_i=-1$, and that $m = cn$ for some $c\geq 1$, so that 
	$w^*\geq 0$. The argument for $w^*<0$ is identical with signs reversed.  
	If $c=1$, then $w^*=0$, the result is immediate as $\lambda^* = L$, 
	and so convergence follows as $\eta < 2/L$. 
	Let $\sigma(z)=1/(1+e^{-z})$ denote the sigmoid function. Then 

	\begin{minipage}{0.56\textwidth}
		\begin{align}
			\label{eq:1d-grad-expansion}
			\loss'(w) &= \sigma(w) - \frac{m}{n+m} = \sigma(w) - \frac{c}{c+1}.
		\end{align}
		Stationarity at $w^*$ gives us $\sigma(w^*) = c/(c+1)$, which implies 
		$w^* = \ln(c)$.
		The Hessian at $w^*$ is then $\lambda^* = \loss''(w^*) = \sigma'(w^*)$,
		and the GD map is simply
		\begin{align}
			T(w) 
			= w - \frac{\gamma}{\sigma'(w^*)} (\sigma(w) - \sigma(w^*)).
		\end{align}
	\end{minipage}
	\hfill
	\begin{minipage}{0.4\textwidth}
		\begin{center}
			\includegraphics[width=\textwidth]{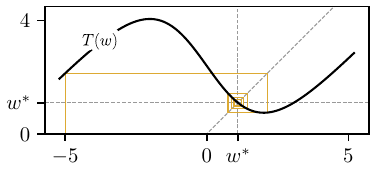}	
		\end{center}
		\captionof{figure}{ Cobweb diagram of $T$ for $c=3$ and $\gamma=1.8$. }
		\label{fig:1d-map-illustration}
	\end{minipage}

	To understand the shape of $T$, note that
	$T'(w) = 1 - \gamma \sigma'(w) / \sigma'(w^*)$.
	Since $\sigma'$ is symmetric about $0$ and strictly decreasing on
	$(0,\infty)$, $T'=0$ at exactly two points
	$w_\ell=-w_r$ and $w_r>w^*$, with $T'<0$ between them and $T'>0$
	outside. Thus, $w_\ell$ is a local maximum and $w_r$ is a local minimum
	of $T$ (see \cref{fig:1d-map-illustration}). In particular,
	$T'(w^*)=1-\gamma\in(-1,0)$,
	so any nontrivial trajectory converging to $w^*$ must eventually alternate
	between its two sides. This oscillatory convergence is reminiscent of 
	that of the logistic map for parameters in the range $(2,3)$, see 
	for instance \citet[Section 2.1.6]{sternberg2010dynamical}. 
	Our argument relies on three key steps, which we describe here and prove in 
	\cref{sec:app-1d-sphere-convergence-proofs}. 

	By \cref{lemma:contraction-from-right}, every step from $w>w^*$
	contracts the distance to $w^*$, even if it crosses to the left of $w^*$,
	while a step from the left need not contract. Nonetheless,
	\cref{lemma:double-crossing-rlr,lemma:two-step-contraction-from-right}
	show that whenever $w_t>w^*$ and $w_{t+1}<w^*$, the next iterate
	returns to the right and 
	\begin{align}
		\label{eq:two-step-contraction}
		\frac{w_{t+2} - w^*}{w_t - w^*} \leq 1-\gamma(2-\gamma) \frac{\bar\lambda(w_t)^2}{\sigma'(w^*)^2} < 1,
	\end{align}
	where $\bar\lambda(w) \coloneqq (\sigma(w) - \sigma(w^*))/(w - w^*)$
	is the average Hessian between $w$ and $w^*$.
	If a trajectory remains on one side of $w^*$, the sign of $\loss'$
	moves it monotonically toward $w^*$. Since $T'(w^*)<0$, it must
	eventually make a cross unless it reaches $w^*$ exactly. 
	Thus regardless of the initialization, the trajectory must 
	enter this right-to-left-to-right oscillation.
	Each return to the right is closer to $w^*$. The right iterates 
	$w_t>w^*$ remain in a compact interval on which
	$\bar\lambda$ is bounded away from zero, so the contraction
	factor is uniformly less than one. Hence the right iterates
	converge to $w^*$. As $T$ is continuous, the left iterates 
	must also converge, hence global convergence 
	of the entire trajectory.
\end{proof}
\paragraph{Rate of convergence}
Let $\tilde w>w^*$ be the unique point satisfying
$T(\tilde w)=w^*$, so that $T(w_t)<w^*$ for all
$w_t\in(w^*,\tilde w)$. The preceding proof shows that the trajectory
enters this oscillatory neighbourhood in finite time. 
As $\tilde w - w^* = \eta (\sigma(\tilde w) - \sigma(w^*))$,
we have $\bar\lambda(\tilde w)=1/\eta=\sigma'(w^*)/\gamma$.
Recursing on \labelcref{eq:two-step-contraction} gives
\begin{align}
	w_{t+2k} - w^*
	&< \paren{2(\gamma-1)/\gamma}\smash{^{k}}(w_t-w^*)
\end{align}
for all integer $k\geq 1$.
This loose bound of the rate already shows that as $\gamma\to 2$,
the two-step contraction factor approaches $1$. This reflects 
the approach to the period-doubling bifurcation point at $\gamma=2$,
where $w^*$ loses stability and a stable $2$-cycle emerges. 
Surprisingly, although the iterate distance
contracts only every two steps, the gradient norm actually decreases
monotonically. This is formalized in
\cref{thm:1d-convergence-sphere-grad-norm} as a consequence of
\cref{prop:one-dimensional-gradient-contraction}, which we use
repeatedly in the remainder of this work. 
Essentially, for a twice-differentiable,
strictly convex 1D objective with positive curvature 
at its minimizer, ruling out a $2$-cycle is sufficient 
to guarantee global contraction of the gradient norm 
and local linear convergence. In fact, a short 
argument shows that the $2$-step distance contraction
follows directly from this result, but the 
standalone proof makes the crossings more explicit and intuitive. 

So far we have shown that the spherical assumption is sufficient to 
extend the 1D global convergence guarantee from $\gamma\leq 1$ to 
the full local stability range $\gamma<2$, and 
the rate is linear once we enter the two-step contraction neighbourhood. 
Unfortunately, as we show next, global convergence fails in higher 
dimensions: for every $\gamma<2$, there exists a spherical dataset 
that can still produce cyclic GD trajectories.

%% file: sections/high-d-counterexample.tex
\section{Counterexamples in higher dimensions}
The first counterexample is an implicit construction 
that embeds an implicit 2D cycle from \citet{meng2024gradient} 
onto a very high dimensional sphere.
\begin{figure*}
	\centering \includegraphics[width=0.9\textwidth]{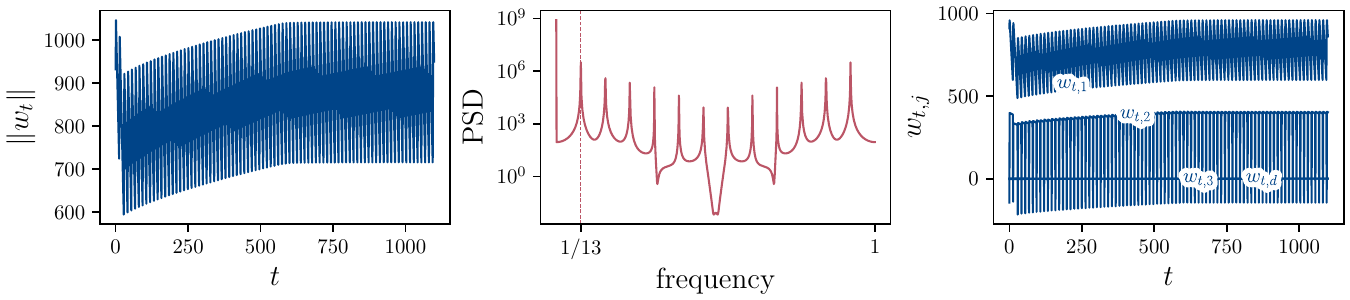}	
	\caption{ GD converges to a cycle on a $d=5000$ dimensional 
	dataset on the unit sphere for $\gamma=0.95$. We ran GD for a total 
	number of $T=20000$ iterations. The first panel shows the norms of the 
	iterates, and the last are the actual values of the iterates at a few 
	different coordinates. For these panels we only showed the first $1000$ 
	iterations so that convergence to the cycle can be viewed clearly. 
	The second panel shows the power spectrum computed using the last 
	$1024$ iterations of $\norm{w_t}$, indicating a period-$13$ cycle.}
	\label{fig:counterexample-on-sphere-high-d}
\end{figure*}
\begin{restatable}{thm}{highdspherecounterexample}
    For every $\gamma\in(0,2)$, there exists a dimension $d>2$,
    a non-separable logistic regression dataset in $\R^d$ satisfying
    $\norm{x_i}=1$ for all $i$, and an initialization $w_0$ such that
    GD with step size $\eta=\gamma/\lambda^*$ converges to a stable
    cycle of period $k>1$.
\end{restatable}
\begin{proof}
Fix $\gamma\in(0,2)$. By \citet[Theorem 3]{meng2024gradient},
there exists a non-separable dataset in $\R^2$ on which GD with
step size multiplier $\gamma$ converges to a stable cycle of period
$k>1$. Denote this base dataset using 
$X_b\in\R^{n_b\times 2}$ where $n_b$ is the number of examples, and all 
labels are $1$'s. 
Assume that each $x_i$ in this dataset is already scaled to have norm 
at most $1$, which can be achieved by scaling all examples by 
the norm of the largest. 
\cref{prop:gd-preserved-under-scaling} guarantees that the cycle is
preserved for the same $\gamma$.
The base objective is denoted as
\begin{align}
	\loss_b(w) = \frac{1}{n_b} \sum_{i=1}^{n_b} \ell( -w\transpose x_i ),
\end{align}
where $w^*_b$ is the solution, and $\ell(\cdot) = \ln(1+\exp(\cdot))$ is 
the logistic loss. 
Let $\lambda_b^* \coloneqq \lambda_{\max}(\nabla^2\loss_b(w_b^*))$. 
Our spherical counterexample will be constructed by ``schmearing'' 
this base dataset onto a high-dimensional unit sphere,
while preserving the cycle in the $2$-dimensional subspace. 
Let \smash{$s_i \coloneqq \sqrt{1-\norm{x_i}\!{}^2}$}. 
Now for some $d>2$, 
we duplicate and pad each example in the base dataset as follows:
for $i=1,\dots, n_b$, 
\begin{align*}
	\begin{pmatrix}
		x_i \\ 
		s_i e_1
	\end{pmatrix}, \dots ,
	\begin{pmatrix}
		x_i \\ 
		s_i e_{d-2}
	\end{pmatrix},
	\begin{pmatrix}
		x_i \\ 
		-s_i e_1
	\end{pmatrix}, \dots ,
	\begin{pmatrix}
		x_i \\ 
		-s_i e_{d-2}
	\end{pmatrix} 
\end{align*}
where $e_j$ is the $j$-th coordinate vector in $\R^{d-2}$,
and all corresponding labels are $1$'s as well.
The resulting dataset contains $n=2(d-2)n_b$ examples, each with unit
norm. Denote each of these new examples 
by $\tilde x_{i'}$. This yields a new objective $\loss$ 
minimized at $w^* = (w^*_b, \, \zerovec_{d-2})\transpose$
as the new loss is symmetric in the last $d-2$ coordinates.
We can verify this by evaluating
\begin{align*}
	\nabla \loss(w^*) 
	&= \frac{1}{n} \sum_{i=1}^{n_b} -\ell'(- x_i\transpose w^*_b) \sum_{j=1}^{d-2} \paren{ \begin{pmatrix}
		x_i \\ s_i e_j 
	\end{pmatrix} - \begin{pmatrix}
		-x_i \\ s_i e_j
	\end{pmatrix} } \\
	&= \frac{1}{n_b} \sum_{i=1}^{n_b} -\ell'(-x_i\transpose w^*_b) \begin{pmatrix}
		x_i \\ \zerovec_{d-2}
	\end{pmatrix} \\
	&= \zerovec_d,
\end{align*}
Furthermore, the Hessian is given by 
\begin{align*}
	\nabla^2 \loss(\tilde w) &= \frac{1}{n} \sum_{i'=1}^{n} \ell''(-\tilde w\transpose \tilde x_{i'}) \tilde x_{i'}\tilde x_{i'}\transpose,
\end{align*}
which has a block-diagonal structure. At $w^*$, 
its top-left block is $\nabla^2 \loss_b(w_b^*)$, and its bottom-right 
block is $c_b/(d-2) \cdot I_{(d-2)\times (d-2)}$, where 
\begin{align}
	c_b \coloneqq \frac{1}{n_b} \sum_{i=1}^{n_b} \ell''(-x_i\transpose w_b^*) s_i^2.
\end{align}
This constant depends only on the base dataset.
Most importantly, it is independent of $d$. 
Therefore $\lambda^* = \lambda_{\max}(\nabla^2\loss(w^*))$ 
will either equal to $\lambda_b^*$ 
or $c_b/(d-2)$, whichever is larger. So to suppress the bottom right 
block, we can take $d$ to be large enough so that 
\begin{align}
	\label{eq:dimension-requirement}
	\lambda_b^* \geq c_b / (d-2)
\end{align}
to preserve the base curvature. This is certainly achievable with a finite $d$.
Therefore, GD with $\eta=\gamma/\lambda^*$ restricted to the 
first two dimensions coincides with GD on the base dataset 
and has the same period-$k$ cycle.
It remains to verify stability in added dimensions. 
By symmetry, the Jacobian of the GD map $J_wT$ at 
any point $w_{k'}=(u_{k'},\zerovec_{d-2})$ on the cycle 
inherits the block diagonal form of the Hessian. Specifically, 
the bottom-right block is $(1 - \eta c_{k'} / (d-2)) I_{d-2}$, where 
$c_{k'}$ is again independent of $d$. 
The top-left block is the Jacobian of the base dynamics, 
whose cycle is stable. Since the cycle contains only
finitely many points, we can increase $d$ further if needed so that
$\abs{ 1 - \eta c_{k'}/(d-2) } < 1$ for all $k'=1,\dots,k$.
The product of these factors over one period also has magnitude less
than one. Hence the cycle is stable in the additional coordinates and,
therefore, in $\R^d$.
\end{proof}
Numerically, we took a two-dimensional counterexample from 
\citet{meng2024gradient} and followed the schmearing steps above
to create a $d=5000$-dimensional dataset on the unit sphere,
while ensuring \cref{eq:dimension-requirement} holds. 
We verified that GD indeed converges to a cycle using the prescribed 
step size for some $\gamma<2$ (that generated the cycle in the base 
dataset), as seen in \cref{fig:counterexample-on-sphere-high-d}.

\subsection{An explicit counterexample in $d=2$}
\label{sec:two-dim-counterexample}
While the previous counterexample shows that for all $\gamma\in(0,2)$, 
we can construct a spherical dataset in \emph{some} 
dimension $d>2$ such that convergence 
to a limit cycle is possible, the schmearing construction necessitates
a fairly large $d$. It is therefore natural to ask: 
do spherical counterexamples exist in lower dimensions? We now
show that this is true by providing an explicit counterexample in $d=2$.
\begin{restatable}{prop}{twodimsphericalcounterexample}
\label{prop:2d-spherical-counterexample}
    Let $a=7/25$ and $b=24/25$.
    Consider a dataset in $\R^2$ consisting of
    $800$ copies of $e_1$,
    $1$ copy of $-e_1$,
    $4$ copies of $(a,b)$,
    and $4$ copies of $(a,-b)$, for a total of $n=809$ examples
    all with labels $y_i=1$. Every $x_i$ has unit norm, the dataset is
    non-separable, and the corresponding logistic regression objective
    has a unique minimizer $w^*$. Under the step size
    $\eta=\gamma/\lambda^*$ with $\gamma\approx1.88<2$,
    GD can converge to a stable $2$-cycle.
\end{restatable}
\begin{figure}[h]
	\centering \includegraphics[width=0.65\linewidth]{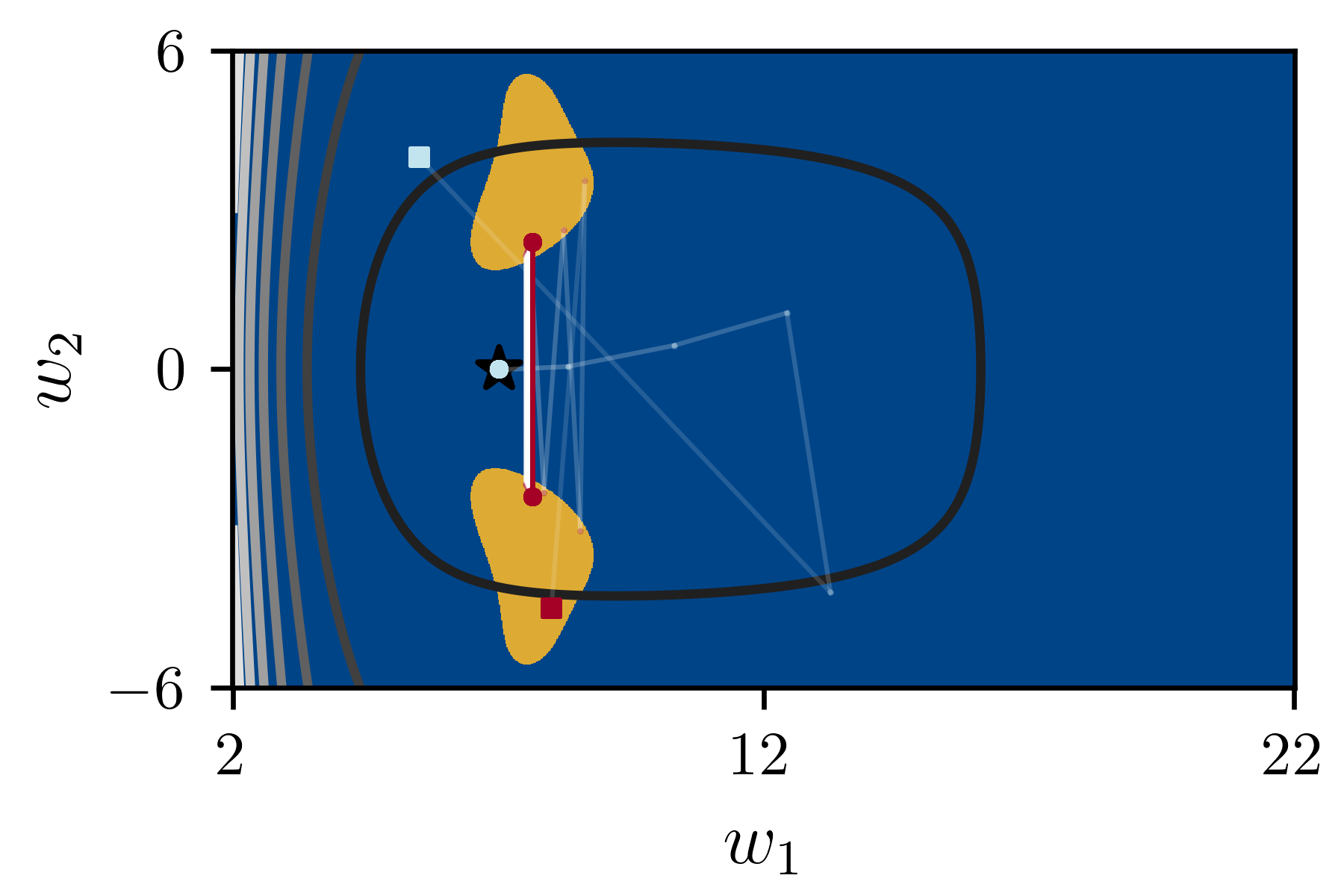}	
	\caption{Loss contour and basins of attraction of the two-dimensional 
	counterexample in \cref{prop:2d-spherical-counterexample}. 
	The minimizer is marked by the star.
	The dark blue region corresponds to the basin of attraction of $w^*$, 
	whereas the light yellow region is that of the stable $2$-cycle. 
	Two trajectories with respective initializations (marked in square) 
	in the two basins are also shown. The last 
	20 iterations of the $2$-cycle are marked in red.}
	\label{fig:2d-sphere-counterexample}
\end{figure}
The proof can be found in \cref{sec:app-misc}, and a numerical 
verification of this cycle (along with its basin of attraction)
is in \cref{fig:2d-sphere-counterexample}.
This counterexample involves a dataset on the two-dimensional unit sphere 
that is, although symmetric about the $x$-axis, highly skewed. 
The skewness comes not from 
the norms of the examples, but rather the sheer number of 
copies of a single example. 
The mechanism underlying this cycle is also different from that of 
the cycles constructed in \citet{meng2024gradient}: 
those cycles used a few large-magnitude examples that 
remain dormant for most of the trajectory until they become active,
providing the kicks needed to close the cycle. Here, the $2$-cycle 
here is maintained by the upper and lower slanted examples. 
At the upper cycle point $w_1=(u,v)$, 
the lower slanted examples have smaller logits $w_1\transpose x_i$, 
contributing a larger positive vertical component to the gradient.
This then drags the next iterate down to $w_2=(u,-v)$.
At this lower point, reflection symmetry swaps their roles and sends 
the iterate back to $w_1$. Meanwhile, the horizontal gradient components 
vanish at both cycle points: the combined horizontal contribution of the 800 
copies of $e_1$ and a single copy of $-e_1$ exactly cancels the 
horizontal contribution of the slanted examples.

%% file: sections/high-d-convergence.tex
\section{Global convergence on the population objective}
\label{sec:population-obj}

Both the explicit 2D counterexample and the
high-dimensional schmearing construction show that restricting examples
to the sphere is not sufficient, 
as their distribution also plays a crucial role. 
We now show that when features are uniformly distributed on the sphere,
GD on the \emph{population objective} converges globally for every step size
up to the stability threshold. We first present the setup and some 
useful properties.

Let $X\sim\mathrm{Uniform}(\sphere^{d-1})$.
The distribution of the random label $Y$ conditioned 
on $X=x$ can be arbitrary.
The population objective is given by
\begin{align}
    \label{eq:population-objective}
    \loss(w) = \E\paren{\ln\paren{1+\exp(-Yw\transpose X)}}.
\end{align}
For any label $Y\in\braces{-1,1}$ and $\alpha\in\R$, it holds that
$\ln\paren{1+\exp(-Y\alpha)} = \ln\paren{1+\exp(\alpha)} - \alpha(1+Y)/2$.
Since $\E\paren{X}=0$, the objective is then equivalent to
\begin{align}
    \label{eq:population-objective-label-correlation}
    \loss(w) &= \E\paren{\ln\paren{1+\exp(w\transpose X)}}-
    w\transpose\E\paren{YX}/2.
\end{align}
Let $U$ be the first coordinate of $X$ and $\phi:\R\to\R$ be 
\begin{align}
    \label{eq:population-phi}
    \phi(r) \coloneqq \E\paren{\ln\paren{1+\exp(rU)}}.
\end{align}
Note that for any $r > 0$, 
\begin{align*}
	\phi'(r) &= \E \paren{ \sigma(rU) U }, \\
	\phi''(r) &= \E \paren{\sigma'(rU) U^2} > 0 \\
	\phi'''(r) &= \E\paren{ \sigma''(rU) U^3 } < 0.
\end{align*}
As $X$ is rotationally invariant, any fixed unit vector $\hat w$ satisfies
$\hat w\transpose X \overset{d}{=} e_1\transpose X = U$, therefore
\begin{align}
    \label{eq:population-objective-radial}
    \loss(w) &= \E\paren{\ln\paren{1+\exp(\norm{w}U)}} - w\transpose\E\paren{YX} / 2\nonumber \\
	&= \phi(\norm{w}) - w\transpose\E\paren{YX} / 2.
\end{align}
Let $v$ be the unit vector in the direction of $\E(YX)$ 
when it's nonzero, and any unit vector otherwise. 
Observe that the direction of $w$ (and the label distribution) only affects 
the second term in $\loss$. Over all $w$ at a fixed radius $r=\norm{w}$, 
this term is minimized when $w$ points in the direction of $v$, 
and its minimum value is $-r\norm{\E(YX)}/2$. Thus if $\loss$ admits 
a finite minimizer it must be parallel to $v$. Furthermore,
\begin{align}
    \norm{\E\paren{YX}} &= \E\paren{Yv\transpose X} \leq \E\paren{\abs{v\transpose X}} = \E\paren{\abs{U}}
\end{align}
with equality if and only if
$Y=\sign(v\transpose X)$ almost surely, in which case the
population problem is (non-strictly) separable. In this case, 
global directional convergence holds nonetheless for all step sizes $\eta>0$ 
via a simple argument in \cref{prop:pop-obj-convergence-separable}. 

We will therefore focus on the strict inequality (non-separable) case.
To find the expression for $w^*$, first note that 
$\E(\abs{U})/2 = \lim_{r\to\infty}\phi'(r)$.
Since $\phi'$ is strictly increasing on $r>0$,
there exists a unique and finite $\beta \geq 0$ such that
\begin{align*}
    \phi'(\beta) = \frac{1}{2}\norm{\E\paren{YX}} \iff \frac{1}{2}\E\paren{YX} = \phi'(\beta)v.
\end{align*}
Thus the objective \labelcref{eq:population-objective-radial} becomes
\begin{align}
	\label{eq:population-objective-nonseparable-radial}
    \loss(w)=\phi(\norm{w})-\phi'(\beta)v\transpose w
\end{align}
and $\beta$ can be interpreted as a signal strength induced 
by the label distribution. Since the minimizer 
must be parallel to $v$, it remains only to find its norm. 
Along the ray $r\mapsto w=rv$, the loss
$\loss(rv)=\phi(r)-\phi'(\beta)r$
is uniquely minimized at $r=\beta$, which gives
\begin{align}
    w^*=\beta v.
\end{align}
Lastly, we need to find $\lambda^*$. Since GD is rotationally invariant, 
we can assume without loss of generality that
$v=e_1$. Now define $\lambda:\R_+\to \R_+$ to be 
\begin{align}
    \label{eq:population-lambda}
    \lambda(r) \coloneqq \frac{\phi'(r)}{r},
\end{align}
with the continuous extension to $\lambda(0)=\phi''(0)$. Writing
$w=\norm{w}\hat w$ gives us
\begin{align*}
    \nabla\loss(w) &= \phi'(\norm{w})\hat w-\phi'(\beta)e_1 = \lambda(\norm{w})w-\phi'(\beta)e_1,\\
    \nabla^2\loss(w) &= \phi''(\norm{w})\hat w\hat w\transpose+\lambda(\norm{w})\paren{I-\hat w\hat w\transpose},
\end{align*}
with their respective continuous extensions at $w=0$.
So there are at most two distinct eigenvalues (for $d\geq 2$), which are
$\phi''(\norm{w})$ in the direction of $w$, and $\lambda(r)$
in every direction orthogonal to $w$.
To determine which one is larger, recall that $\phi'(0)=0$ and 
$\phi''$ is strictly decreasing on $(0,\infty)$. Therefore for all $r>0$,
\begin{align*}
	\lambda(r) = \frac{\phi'(r) - \phi'(0)}{r} = \frac{1}{r} \int_0^r \phi''(s) \,d s > \phi''(r),
\end{align*}
and the two eigenvalues are always ordered as
\begin{align}
	\label{eq:eigval-order}
	0 < \phi''(r) < \lambda(r).
\end{align}
This also implies that $\lambda'(r) = (\phi''(r) - \lambda(r))/r < 0$, and
so $\lambda$ is decreasing on its domain. At $w^* = \beta e_1$, 
\begin{align}
	\lambda^* = \lambda_{\max} \paren{\nabla^2 \loss(w^*)} = \lambda(\beta) \leq \lambda(0) = L,
\end{align}
hence a nontrivial gap between $L$ and $\lambda^*$.
We are now ready to state our main result, the full proof of which can be found 
in \cref{sec:app-population-obj-proof}.
\begin{restatable}{thm}{highdpopconvergence}
    \label{thm:pop-obj-convergence}
	Suppose the population objective \labelcref{eq:population-objective}
	has a unique and finite minimizer. Then for any $w_0\in\R^d$,
	GD iterates with step size $\eta=\gamma/\lambda^*$
	converges to $w^*$ for all $\gamma \in(0,2)$, and the local 
	convergence rate is linear.
\end{restatable}
\begin{proof}[Proof sketch]
	The case $d=1$ reduces to \cref{thm:1d-convergence-sphere}
	(shown in \cref{lemma:population-1d-reduction}).
	The spherical distribution assumption lets us 
	reduce the GD dynamics to essentially two dimensions. 
	Let $r = \norm{w}$ and
	\begin{align*}
		w = (x, z)
	\end{align*}
	where $x = w_1\in\R$ is the signal component and $z \in \R^{d-1}$
	is the orthogonal component. Let $q:\R_+\to\R$ be
	\begin{align*}
		q(r) \coloneqq 1 - \eta \lambda(r).
	\end{align*}
	Since $\eta\phi'(\beta) = \eta\beta\lambda^* = \gamma\beta$, 
	the GD update $T$ (using the gradient of \labelcref{eq:population-objective-nonseparable-radial})
	can be decomposed exactly into 
	\begin{align}
		\label{eq:decomposed-updates-sphere-pop-obj}
		T_x(w) &= q(r) x + \gamma\beta \nonumber \\
		T_z(w) &= q(r) z. 
	\end{align}
	At the solution, $x^* = \beta$ and $z^* = 0$.
	To show convergence, we first prove the existence of a radius $R>0$ 
	such that $\mathcal{B}_R \coloneqq \braces{w:\, \norm{w} \leq R}$ is 
	forward invariant. That is, $\norm{T(w)} \leq R$ 
	whenever $\norm{w} \leq R$. Moreover, every trajectory 
	must enter $\mathcal{B}_R$ after finitely many iterations.

	Once inside $\mathcal{B}_R$, we show that for some 
	parameter $\theta>0$ (depends only on $\beta$, $\gamma$, and $d$),
	the weighted orthogonal component error
	\begin{align}
		B_\theta(w) = \norm{z} \exp(\theta(\beta - x)_+)
	\end{align}
	contracts at a rate $\rho_z<1$, namely,
	\begin{align*}
		B_\theta(T(w)) \leq \rho_z B_\theta(w).
	\end{align*}
	Importantly, the unweighted $\norm{z}$ need not decrease 
	at every step. Since $\eta\lambda(\beta) = \gamma < 2$ and 
	$\lambda(r)$ is decreasing, $\abs{q(r)} > 1$
	can only occur when $r<\beta$. Hence, whenever $\norm{z}$ 
	grows, it must be that $x<\beta$, and taking a step increases $x$ 
	towards $\beta$. \cref{lemma:orthogonal-direction-decay}
	shows that this progress in $x$ offsets the growth in $\norm{z}$
	in $B_\theta$.

	For the signal component $x$, we analyze the gradient along the 
	one-dimensional axis where $z=0$, that is,
	\begin{align}
		g(x) = \frac{d}{dx} \loss((x,0)).
	\end{align}
	Along this axis, $|g(x)|$ contracts linearly at every iteration,
	a result that again follows from \cref{prop:one-dimensional-gradient-contraction}
	which we used to show gradient norm monotonicity on 1D spherical data 
	(\cref{sec:1d-convergence,thm:1d-convergence-sphere-grad-norm}). 
	Here, the condition $\gamma<2$ prevents this one-dimensional 
	update from overshooting far enough to reach a point 
	on the other side of $w^*$ with opposite gradient, 
	preventing a $2$-cycle. The actual update 
	of the signal component differs from the update along this axis 
	by $O(\normsq{z})$, 
	\begin{align*}
		\abs{g( T_x(w) )} \leq \rho_x \abs{g(x)} + CB_\theta(w)^2
	\end{align*}
	for some $C>0$ and $\rho_x<1$. 
	Finally, choosing $\vartheta>0$ sufficiently large, the Lyapunov function
	\begin{align*}
		V(w) = \abs{g(x)} + \vartheta B_\theta(w)^2
	\end{align*}
	can be shown to decrease at a linear rate for all $w\in\mathcal{B}_R$. 
	As $V(w) = 0$ if and only if $w= w^*$,
	global convergence with a local linear rate is guaranteed. 
\end{proof}
\paragraph{Local asymptotic rate}
Let $\mu^* \coloneqq \lambda_{\min}(\nabla^2\loss(w^*))$ 
and $\kappa^* \coloneqq \lambda^*/\mu^*$
be the local condition number. Linearizing the GD map at $w^*$ gives 
\begin{align*}
	\limsup_{t\to\infty} \norm{w_t - w^*}^{1/t} &\leq \max\braces{\abs{1-\gamma},\, \abs{1-\gamma/\kappa^*}} < 1.
\end{align*}
For $d\geq 2$, the two terms correspond to the orthogonal direction
and the signal direction, with $\mu^* = \phi''(\beta)$ and 
$\lambda^* = \lambda(\beta)$. Minimizing this over $\gamma$ yields
\begin{align*}
	\gamma_{\text{opt}} = \frac{2\kappa^*}{\kappa^*+1}, 
	\quad 
	\eta_{\text{opt}} = \frac{2}{\lambda^* + \mu^*},
	\quad 
	\rho_{\text{opt}} = \frac{\kappa^*-1}{\kappa^*+1}.
\end{align*}
These quantities are analogous to the global optimal tuning and rate for GD on 
smooth and strongly convex objectives \citep{de2017worst,lessard2016analysis}
with the curvature constants taken at the minimizer rather than globally. 
In particular, for $\gamma \geq \gamma_{\text{opt}}$, 
we have $\rho^*(\gamma) = \gamma-1$, so the worst case local rate 
deteriorates as $\gamma\to 2$, as expected.

\paragraph{Isotropic Gaussian distribution}
One may ask if the same population result should also hold for 
other rotationally-invariant distributions like 
centered isotropic Gaussian, where $X\sim\mathcal N(0,s^2I_d)$, 
with the same arbitrary label distribution 
and $\gamma$ range. After rescaling, rotational invariance 
gives the same decomposed GD update. The only part of the
proof that requires substantial change would be 
in \cref{lemma:vareps-inequality}.
If true, this would complement \citet{chen2026finite}, 
which established finite-sample estimation guarantees and 
faster local convergence under large step-sizes for GD on 
logistic regression, where $X\sim\mathcal N(0,s^2I_d)$ 
and the labels follow the conditional logistic distribution. 
However, their focus is around the statistical properties 
of the optimization path and requires initializing near the ground truth, 
whereas we study global convergence under the spherical distribution without 
imposing assumptions on the label distribution.
Nonetheless, it is unlikely that the extension of our result to the 
isotropic Gaussian case would immediately cover the more interesting
anisotropic setting.

\paragraph{Finite sample analysis}
This section studies the population loss with features drawn uniformly 
from the unit sphere to rule out both large-magnitude examples 
and highly irregular feature distributions that can produce cycles. 
A natural next step is to ask whether GD also converges globally 
for all $\gamma\in(0,2)$ on finite datasets sampled from this distribution. 
One would expect such a result to hold with high probability 
when the sample size is large enough for the derivatives of 
the empirical objective to approximate those of the population objective 
uniformly. The more interesting questions, we believe, are how the 
required sample size and probability of convergence depend on the dimension $d$
and signal strength $\beta$, as well as whether it's actually true 
for all $\gamma\in(0,2)$. 

%% file: sections/1d-suff-cond.tex
\section{Beyond the sphere}
Recall that on any finite dataset, the spherical assumption is only sufficient
for global convergence in one dimension, while counterexamples 
in higher dimensions. In this section, we therefore 
return to the one-dimensional setting and propose 
a relaxation of the spherical assumption that guarantees 
global convergence for all step sizes up to the stability threshold.
We assume all labels are $1$ and discard all examples 
where $x_i=0$, as such examples contribute
only an additive constant of $\ln(2)/n$ to the objective. Define 
$r_i \coloneqq \abs{x_i}$, $r_{\max} \coloneqq \max_i r_i$, 
$r_{\min} \coloneqq \min_i r_i$, and 
\begin{align*}
	\nu \coloneqq r_{\max} / r_{\min}.
\end{align*}
Without loss of generality, assume $r_{\max} = 1$ as uniformly scaling 
the examples by the same constant preserves GD dynamics 
(\cref{prop:gd-preserved-under-scaling}). The quantity $\nu$ 
essentially measure how spread-out the examples are in terms of their magnitude.
We will prove that as long as $\nu\leq 4$, or equivalently,
if they lie on an annulus in one dimension, then
global convergence is guaranteed for all $\eta < 2/\lambda^*$. 

\begin{restatable}{thm}{onedimsufficientcondition}
Let $d=1$. Assume the dataset is non-separable and that $\nu\leq 4$. 
Then for any $w_0\in\R$, GD iterates with 
step size $\eta=\gamma/\lambda^*$ converges to $w^*$ for all $\gamma\in(0,2)$.
In particular, for all $w\neq w^*$, 
\begin{align*}
	\abs{\loss'(T(w))} < \abs{\loss'(w)}.
\end{align*}
Consequently, every trajectory converges to $w^*$. Moreover, 
for all $\gamma\in(1,2)$, the iterates enter a local invariant neighbourhood
$\gN_\gamma$ of $w^*$ in finite time, within which we have 
$\abs{\loss'(T(w))} \leq (\gamma/2) \abs{\loss'(w)}$.
\end{restatable}
The proof is in \cref{sec:app-1d-suff-cond-proof}.
As this result shows, the spherical assumption can be relaxed to 
an annulus condition in $d=1$ to yield the exact 
same convergence guarantee. It turns out this condition is sufficient 
to prevent a large GD from reaching a point on the opposite side 
with the same gradient magnitude, thereby avoiding a $2$-cycle. 
This is a recurring pattern in all one-dimensional scenarios we have 
seen, where \cref{prop:one-dimensional-gradient-contraction} has 
been applied. The constant $4$ is likely not tight, but it suffices to show 
that controlling how spread-out the example magnitudes are can guarantee 
global convergence up to the stability threshold in $d=1$.
It is unclear what would be higher-dimensional analogue of this result, 
as given our spherical counterexamples, we already have 
to rule out a deterministic, finite-sample argument. One
plausible route would again be via the population loss on a distribution 
uniform on the higher-dimensional annulus.

%% file: sections/conclusion.tex
\section{Conclusion}

In this work, we have shown that when restricting data to be 
on the sphere, global convergence is guaranteed for all $\gamma<2$ 
if $d=1$, but this is not true in general in higher dimensions 
on arbitrary finite-sample datasets. Most importantly, this is not 
just a phenomenon in very high dimensions, as even in $d=2$ 
counterexamples exist. However, if the samples follow a 
uniform distribution on the sphere, then on the population loss, 
GD converges globally for all step sizes up to the stability threshold. 
We further relax the assumption in one dimension by allowing 
the examples to lie on an annulus rather than 
exactly on the sphere, and show that the desired convergence behaviour 
still holds. Thus, we have, to some extent, answered the question
about how common these pathological behaviour are on datasets that 
are more or less ``regular''. Interesting follow up questions 
would be to generalize our population objective result to 
non-isotropic Gaussians, as well as a careful analysis of 
convergence on finite-sample datasets from these distributions. 
It would also be interesting to generalize the annulus condition 
to higher dimensions to bring the theory closer to more realistic datasets. 
Addressing these questions would offer deeper 
insight into the GD dynamics on logistic regression, 
and potentially on training neural network classifiers.

%% file: sections/appendix/appendix.tex
\renewcommand{\contentsname}{Appendix}
\tableofcontents
\vspace{2em}
\addtocontents{toc}{\protect\setcounter{tocdepth}{3}}

\input{sections/appendix/app-1d-sphere.tex}
\input{sections/appendix/app-sphere-distr.tex}
\input{sections/appendix/app-1d-suff-cond.tex}
\input{sections/appendix/app-misc.tex}

%% file: sections/appendix/app-1d-sphere.tex
\section{Proofs for one-dimensional spherical data}
\label{sec:app-1d-sphere-convergence-proofs}

The following three results are used in the proof of 
\cref{thm:1d-convergence-sphere} in \cref{sec:1d-convergence}.

\begin{restatable}[One-step contraction from right to left]{lem}{onestepcontraction}
	\label{lemma:contraction-from-right}
	For all $\gamma\in(1,2)$ and $w > w^*$, $|T(w) - w^*| < |w - w^*|$.
\end{restatable}
\begin{proof}
	If $w\gg w^*$ in that $T(w) > w^*$, then contraction is 
	trivial as we simply moved closer. 
	When we do cross, $T(w) < w^*$, the GD update 
	can be expanded into 
	\begin{align*}
		w^* - T(w) &= w^* - w + \frac{\gamma}{\sigma'(w^*)} \paren{ \sigma(w) - \sigma(w^*) } \\
		\implies \frac{w^* - T(w)}{ w - w^*} &= -1 + \frac{\gamma}{\sigma'(w^*)} \frac{\sigma(w) - \sigma(w^*)}{ w - w^* } \\
		&< \frac{2}{\sigma'(w^*)} \frac{\sigma(w) - \sigma(w^*)}{ w - w^*} - 1 \tag{$\gamma < 2$}\\
		&\leq \frac{2}{\sigma'(w^*)} \max_{\bar w\in[w^*,w]} \sigma'(\bar w) - 1 \\ 
		&= \frac{2}{\sigma'(w^*)} \sigma'(w^*) - 1 \\
		&= 1,
	\end{align*}
	where we used the fact that $w > w^*$ and $\sigma'$ 
	is decreasing for all $w>0$. This gives us contraction whenever we take 
	a step on the right of $w^*$.
\end{proof}

\begin{restatable}[Double crossing]{lem}{doublecrossing}
	\label{lemma:double-crossing-rlr}
	For all $\gamma\in(1,2)$, if $w_t > w^*$ and $w_{t+1} < w^*$, 
	then $w_{t+2} > w^*$.
\end{restatable}
\begin{proof}
	The fact that we crossed over from right to left is equivalent to
	\begin{align}
		w_{t+1} - w^* =  w_t - w^* - \eta (\sigma(w_t) - \sigma(w^*)) &< 0 \nonumber \\
	\iff \eta \bar\lambda(w_t) &> 1 
	\end{align}
	where $R$ is the average average $\sigma'$ between $w$ and $w^*>0$, as in
	\begin{align*}
		\bar\lambda(w) \coloneqq \frac{\sigma(w) - \sigma(w^*)}{ w - w^*},
	\end{align*}
	Moreover, $\bar\lambda(w_{t+1}) > \bar\lambda(w_t)$. This can be proved by a case analysis:
	\begin{enumerate}
		\item If $w_{t+1} > -w^*$, then $\bar\lambda(w_{t+1}) > \sigma'(w^*)$
		as $\bar\lambda$ is the average of $\sigma'$ over $(-w^*, w^*)$, all of which 
		are greater than $\sigma'(w^*)$ due to $\sigma'$ being symmetric and 
		maximized at $0$. Furthermore, $\sigma'(w^*) > \bar\lambda(w_t)$ as 
		$w_t > w^*$ and $\sigma'$ is decreasing on both sides of $0$.
		\item By \cref{lemma:contraction-from-right},
		$w^* - w_{t+1} < w_t- w^*$, 
		implying $w_{t+1} > 2w^* - w_t$. 
		Since $\bar\lambda$ is strictly increasing on $(-\infty, -w^*)$ and 
		$w_{t+1} \leq -w^*$, we must have
		$\bar\lambda(w_{t+1}) > \bar\lambda(2w^* - w_t)$. Moreover, for any $w > w^* > 0$,
		it holds that $\sigma'(w) < \sigma'(2w^* - w)$, so it follows that 
		the average must also be higher 
		over $(2w^* - w, w^*)$ than $(w^*,w)$, giving us 
		$\bar\lambda(2w^* - w_t) > \bar\lambda(w_t)$. Chaining the inequalities 
		yields $\bar\lambda(w_{t+1}) > \bar\lambda(w_t)$ for this case.
	\end{enumerate}
	Together, we have 
	\begin{align*}
		w_{t+2} - w^* &= w_{t+1} - w^* - \eta (\sigma(w_{t+1}) - \sigma(w^*))  \\
		&= w_{t+1} - w^* - \eta \bar\lambda(w_{t+1}) (w_{t+1} - w^*) \\
		&> w_{t+1} - w^* - \eta \bar\lambda(w_t) (w_{t+1} - w^*) \\
		&> w_{t+1} - w^* - (w_{t+1} - w^*) \\ 
		&= 0
	\end{align*}
	and the proof is complete.
\end{proof}

\begin{restatable}[Two-step contraction from the right of $w^*$]{lem}{twostepcontraction}
	\label{lemma:two-step-contraction-from-right}
	For all $\gamma\in(1,2)$, if $w_t > w^*$ and $w_{t+1} < w^*$,
	then on the second step we have 
	\begin{align}
		\frac{w_{t+2} - w^*}{w_t - w^*} \leq 1-\gamma(2-\gamma) \frac{\bar\lambda(w_t)^2}{\sigma'(w^*)^2} < 1.
	\end{align}
\end{restatable}
\begin{proof}
	Observe that 
	\begin{align*}
		\frac{w_{t+2}-w^*}{w_t-w^*} &= \frac{w_{t+2}-w^*}{w_{t+1}-w^*} \frac{w_{t+1}-w^*} {w_t-w^*} \\
		&= \paren{1-\eta\frac{\sigma(w_{t+1}) - \sigma(w^*)}{w_{t+1} - w^*} } \frac{w_{t+1} - w^*}{ w_t-w^*} \\
		&= \paren{1-\eta\frac{\sigma(w_{t+1}) - \sigma(w^*)}{w_{t+1} - w^*} } \paren{ 1 - \eta \frac{\sigma(w_t) - \sigma(w^*)}{w_t - w^*} } \\
		&= (\eta \bar\lambda(w_{t+1}) - 1) (\eta \bar\lambda(w_t)- 1 ).
	\end{align*}
	For some $\epsilon>0$, we can write 
	the iterates as $w_{t+1} = w^* - \epsilon$, 
	which implies $w_t > w^* + \epsilon$ 
	by \cref{lemma:contraction-from-right}. Monotonicity of $\bar\lambda$ 
	on $w>0$ implies $\bar\lambda(w_t) < \bar\lambda(w^* + \epsilon)$. Using this in the above, 
	\begin{align}
		\label{eq:term-in-bracket}
		\hspace{-0.75em}\frac{w_{t+2}-w^*}{w_t-w^*} &\leq (\eta \bar\lambda(w^* - \epsilon) - 1)(\eta \bar\lambda(w^* + \epsilon) - 1) \nonumber \\
		&= 1 - \eta (\bar\lambda(w^* - \epsilon) + \bar\lambda(w^* + \epsilon)) + \eta^2 \bar\lambda(w^* - \epsilon)\bar\lambda(w^* + \epsilon) \nonumber \\
		&= 1 - \gamma \frac{ \bar\lambda(w^* - \epsilon) + \bar\lambda(w^* + \epsilon) }{ \sigma(w^*)(1-\sigma(w^*)) } + \gamma^2 \frac{ \bar\lambda(w^* - \epsilon) \bar\lambda(w^* + \epsilon) }{ ( \sigma(w^*)(1-\sigma(w^*)) )^2 } \nonumber \\
		&= 1 - \gamma \frac{ \bar\lambda(w^* - \epsilon) \bar\lambda(w^* + \epsilon) }{ (\sigma(w^*)(1-\sigma(w^*)))^2 } \brackets{ \frac{ \sigma(w^*)(1-\sigma(w^*)) }{ \bar\lambda(w^*-\epsilon) } + \frac{ \sigma(w^*)(1-\sigma(w^*)) }{ \bar\lambda(w^*+\epsilon) } - \gamma}.
	\end{align}
	It remains to show that the term in brackets is strictly positive. To this end, 
	\begin{align*}
		\sigma(w^* + \epsilon) - \sigma(w^*) &= \frac{ 1 }{ 1 + e^{-(w^* + \epsilon)} } - \frac{ 1 }{ 1 + e^{-w^*} } \\
		&= \frac{ (1 + e^{-w^*}) - (1 + e^{-(w^*+\epsilon)}) }{ (1 + e^{-w^*}) ( 1 + e^{-(w^*+\epsilon)}) } \\
		&= \frac{ e^{-w^*} -  e^{-(w^*+\epsilon)} }{ (1 + e^{-w^*}) ( 1 + e^{-(w^*+\epsilon)}) } \\
		&= \frac{ e^{-w^*} (1 - e^{-\epsilon} ) }{ (1 + e^{-w^*}) ( 1 + e^{-(w^*+\epsilon)}) }.
	\end{align*}
	Additionally, 
	\begin{align*}
		\frac{ 1 }{ \bar\lambda(w^*+\epsilon) } &= \frac{ \epsilon }{ \sigma(w^*+\epsilon) - \sigma(w^*) } \\
		&= \frac{ \epsilon (1 + e^{-w^*}) ( 1 + e^{-(w^*+\epsilon)})  }{  e^{-w^*} (1 - e^{-\epsilon} ) } \\
		&= \epsilon ( 1 + e^{w^*} ) \frac{ 1 + e^{-(w^*+\epsilon)} }{ 1 - e^{-\epsilon} }.
	\end{align*}
	Flipping the sign on $\epsilon$ gives us 
	\begin{align*}
		\frac{ 1 }{ \bar\lambda(w^*-\epsilon) } = -\epsilon (1+e^{w^*}) \frac{ 1 + e^{-(w^*-\epsilon)} }{ 1 - e^{\epsilon} }.
	\end{align*}
	Summing up the two, 
	\begin{align*}
		\frac{ 1 }{ \bar\lambda(w^*+\epsilon) } + \frac{ 1 }{ \bar\lambda(w^*-\epsilon) } &= \epsilon ( 1 + e^{w^*} ) \frac{ 1 + e^{-(w^*+\epsilon)} }{ 1 - e^{-\epsilon} } -\epsilon (1+e^{w^*}) \frac{ 1 + e^{-(w^*-\epsilon)} }{ 1 - e^{\epsilon} } \\
		&= \epsilon ( 1 + e^{w^*} ) \frac{ (1 + e^{-(w^*+\epsilon)})(1-e^{\epsilon}) - (1+e^{-(w^*-\epsilon)})(1-e^{-\epsilon}) }{ (1 - e^{-\epsilon})(1 - e^{\epsilon}) }.
	\end{align*}
	Multiplying out the numerator and re-arranging gives us 
	\begin{align*}
		\frac{ 1 }{ \bar\lambda(w^*+\epsilon) } + \frac{ 1 }{ \bar\lambda(w^*-\epsilon) } &= (1 + e^{w^*}) (1 + e^{-w^*}) \frac{\epsilon (e^{\epsilon} - e^{-\epsilon}) }{ (e^\epsilon - 1) ( 1-e^{-\epsilon}) }.
	\end{align*}
	Observe that $\sigma'(w^*) = \sigma(w^*)(1-\sigma(w^*)) = \sigma(w^*)\sigma(-w^*) = (1+e^{-w^*})^{-1} (1+e^{w^*})^{-1}$, 
	which lets us further simplify the above to just 
	\begin{align*}
		\frac{ \sigma(w^*)(1-\sigma(w^*)) }{ \bar\lambda(w^*+\epsilon) } + \frac{ \sigma(w^*)(1-\sigma(w^*)) }{ \bar\lambda(w^*-\epsilon) } &= \frac{\epsilon (e^{\epsilon} - e^{-\epsilon}) }{ (e^\epsilon - 1) ( 1-e^{-\epsilon}) } \\
		&= \frac{ \epsilon \sinh\epsilon  }{ \cosh\epsilon - 1 } \\
		&= \frac{\epsilon}{\tanh (\epsilon/2)}.
	\end{align*}
	Using the fact that for all $\epsilon>0$, $\tanh(\epsilon) < \epsilon$, 
	\begin{align*}
		\frac{ \sigma(w^*)(1-\sigma(w^*)) }{ \bar\lambda(w^*+\epsilon) } + \frac{ \sigma(w^*)(1-\sigma(w^*)) }{ \bar\lambda(w^*-\epsilon) } > 2.
	\end{align*}
	Putting this back into \cref{eq:term-in-bracket} and substitute back $w_{t+1} = w^* - \epsilon$, 
	\begin{align*}
		\frac{w_{t+2}-w^*}{w_t-w^*} &\leq 1 - \gamma \frac{ \bar\lambda(w^* - \epsilon) \bar\lambda(w^* + \epsilon) }{ (\sigma(w^*)(1-\sigma(w^*)))^2 } \brackets{ \frac{ \sigma(w^*)(1-\sigma(w^*)) }{ \bar\lambda(w^*-\epsilon) } + \frac{ \sigma(w^*)(1-\sigma(w^*)) }{ \bar\lambda(w^*+\epsilon) } - \gamma} \\
		&< 1 - \gamma \frac{ \bar\lambda(w^* - \epsilon) \bar\lambda(w^* + \epsilon) }{ (\sigma(w^*)(1-\sigma(w^*)))^2 } \brackets{ 2 - \gamma } \\
		&= 1 - \gamma (2-\gamma) \frac{ \bar\lambda(w_{t+1}) \bar\lambda(2w^* - w_{t+1}) }{ (\sigma(w^*)(1-\sigma(w^*)))^2 }.
	\end{align*}
	Using $\bar\lambda(w_{t+1}) > \bar\lambda(w_t)$ and $\bar\lambda(2w^*-w_t) > \bar\lambda(w_t)$ 
	from \cref{lemma:double-crossing-rlr}, 
	as well as the assumption that $\gamma<2$, 
	\begin{align*}
		\frac{w_{t+2}-w^*}{w_t-w^*} &\leq 1 - \gamma (2-\gamma) \frac{ \bar\lambda(w_{t})^2 }{ (\sigma(w^*)(1-\sigma(w^*)))^2 } < 1,
	\end{align*}
	proving the right-to-left-to-right two-step contraction, 
	and thus convergence to $w^*$ for any initialization $w_0$.
\end{proof}

\begin{thm}
	\label{thm:1d-convergence-sphere-grad-norm}
	Let $d=1$ and $\norm{x_i}=1$ for all $i$.
	Then for any $w_0\in\R$ and $w_t\neq w^*$, 
	GD iterates with step size $\eta$ of the form 
	\labelcref{eq:step-size} satisfies
	$\abs{\loss'(w_{t+1})} < \abs{\loss'(w_t)}$ for all $\gamma\in(0,2)$.
\end{thm}
\begin{proof}
	As in the proof of \cref{thm:1d-convergence-sphere}, we assume without 
	loss of generality that the ratio between positive and negative 
	examples is $c\geq 1$, so that $w^*\geq 0$. 
	Using the gradient expansion in \labelcref{eq:1d-grad-expansion}
	and letting $p\coloneqq c/(c+1)$, 
	\begin{align*}
		g(w) &\coloneqq \loss'(w) = \sigma(w) - \sigma(w^*) 
		= \sigma(w) - p.
	\end{align*}
	The inverse gradient and its derivative are
	\begin{align*}
		g\inverse(s) &= \ln\frac{p+s}{1-(p+s)} \\
		\psi(s) \coloneqq (g\inverse)'(s) &= \frac{1}{p+s} + \frac{1}{1-(p+s)}.
	\end{align*}
	The latter is strictly convex as $\psi''>0$. 
	Since $g$ is continuously differentiable and strictly increasing, $w^*$ 
	is its unique root, and
	\begin{align*}
		g'(w^*) = \sigma'(w^*) = \lambda^* > 0.
	\end{align*}
	Moreover, since $\eta=\gamma/\lambda^*$ and $\gamma\in(0,2)$,
	we have $0<\eta<2/g'(w^*)$.
	Thus \cref{cor:one-dimensional-gradient-contraction} implies
	\begin{align*}
		\abs{g(w_{t+1})}<\abs{g(w_t)}
	\end{align*}
	for all $w_t\neq w^*$. The same corollary shows that
	$\abs{g(w_t)}\to 0$ for every initialization. Since $g$ is strictly
	increasing with unique root at $w^*$, it follows that $w_t\to w^*$.
\end{proof}

%% file: sections/appendix/app-sphere-distr.tex
\section{Proofs for convergence on the population objective}
\label{sec:app-population-obj-proof}

\subsection{Separable case}
\begin{prop}
	\label{prop:pop-obj-convergence-separable}
	Suppose $X\sim\mathrm{Uniform}(\mathbb S^{d-1})$ and
	$Y=\sign(v\transpose X)$ almost surely for some unit
	vector $v$. Then the classification problem of the 
	population objective \labelcref{eq:population-objective} 
	is linearly separable with the classifier $v$, and
	for any initialization and any GD step size $\eta>0$, it holds that 
	\begin{align*}
		\frac{w_t}{\norm{w_t}} \to v \qquad\text{and}\qquad \loss(w_t)\to 0.		
	\end{align*}
\end{prop}
\begin{proof}
	Recall from the setup in \cref{sec:population-obj} that
	\begin{align*}
		a \coloneqq \frac{1}{2}\E\paren{\abs{U}} = \lim_{r\to\infty}\phi'(r), \quad \text{and} \quad \lambda(r) \coloneqq \frac{\phi'(r)}{r},
	\end{align*}
	where $U$ is the first coordinate of $X\sim\mathrm{Uniform}(\mathbb{S}^{d-1})$.
	Since $Y=\sign(v\transpose X)$ almost surely and we can assume without loss 
	of generality that $v=e_1$, 
	\begin{align*}
		\E(YX) = \E(\sign(U)X) = \E \paren{\abs{U}} e_1 = 2av,
	\end{align*}
	and the objective \labelcref{eq:population-objective-radial} becomes
	\begin{align*}
		\loss(w)=\phi\paren{\norm{w}}-av\transpose w.
	\end{align*}
	Letting $w = (x,z)$ where $x = w_1\in\R$ and $z\in\R^{d-1}$, 
	we can decompose the GD updates into
	\begin{align*}
		x_{t+1} &= x_t - \eta\paren{\lambda(r_t)x_t - a} \\
		z_{t+1} &= \paren{1-\eta\lambda(r_t)}z_t.
	\end{align*}
	We first show that $x_t\to\infty$. If $r_t = \norm{w_t} =0$, then
	$x_{t+1}-x_t=\eta a>0$. For $r_t>0$, we have
	\begin{align*}
		\frac{1}{\eta}(x_{t+1} - x_t) = a-\lambda(r_t)x_t
		&=
		a-\phi'(r_t)\frac{x_t}{r_t} \\
		&=
		\paren{a-\phi'(r_t)}
		+\phi'(r_t)\paren{1-\frac{x_t}{r_t}}
		>0,
	\end{align*}
	where we used $x_t\leq r_t$, $\phi'(0) = 0$, and that $\phi'$ 
	is strictly increasing (and thus $\phi'(r_t)<a$). 
	Therefore $\braces{x_t}$ is strictly increasing.
	Now assume for the sake of contradiction that $\braces{x_t}$ is bounded. 
	Then $x_t$ converges, and hence $x_{t+1}-x_t\to0$. Since both terms in the final expression
	above are nonnegative, this would imply $\phi'(r_t) \to a$ and $x_t/r_t\to 1$.
	However, the first limit requires $r_t\to\infty$, whereas boundedness of $x_t$
	would then give $x_t/r_t\to 0$, a contradiction. Therefore,
	$x_t\to\infty$ and as a result $r_t\to\infty$. Then
	\begin{align*}
		\lambda(r_t)=\frac{\phi'(r_t)}{r_t}\to0,
	\end{align*}
	so after some finite number of iterations it must hold that
	\begin{align*}
		\abs{1-\eta\lambda(r_t)}<1.
	\end{align*}
	Therefore $\norm{z_t}$ must eventually be non-increasing, 
	hence $\braces{z_t}$ must bounded. Since $x_t\to\infty$, we have
	\begin{align*}
		\frac{\norm{z_t}}{x_t}\to 0 \implies \frac{w_t}{\norm{w_t}}\to v.
	\end{align*}
	Finally, let $M\coloneqq \sup_t\norm{z_t} < \infty$. Then 
	\begin{align*}
		Yw_t\transpose X = x_t \abs{v\transpose X} + Yz_t\transpose X \to \infty
	\end{align*}
	almost surely. Therefore loss inside 
	the expectation of $\loss(w) =\E\braces{ln(1+\exp(-Yw\transpose X))}$
	goes to $0$. Once the $x_t$ sequence increases to non-negative, 
	\begin{align*}
		0 \leq \ln(1+\exp(-Yw_t\transpose X)) \leq \ln( 1+\exp(M) ).
	\end{align*}
	Dominated convergence then gives $\loss(w_t) \to 0$
	and the proof is complete.
\end{proof}

\subsection{Non-separable case}
We first show that the $d=1$ case \cref{thm:pop-obj-convergence}
reduces to \cref{thm:1d-convergence-sphere}.
\begin{lem}
	\label{lemma:population-1d-reduction}
	In $d=1$, the population objective \labelcref{eq:population-objective}
	under the distribution of \cref{thm:pop-obj-convergence}
	is equivalent to the one-dimensional objective 
	in \cref{thm:1d-convergence-sphere} where $x_i\in\braces{-1,1}$
	and labels are all $1$'s.
\end{lem}
\begin{proof}
	In one dimension, $X\sim\mathrm{Uniform}(\mathbb{S}^0) \overset{d}{=} \mathrm{Rademacher}$,
	and so the population objective can be simplified as
	\begin{align*}
		\loss(w) &\coloneqq \E\paren{ \ln(1+\exp(-YwX)) } \\
		&= \frac{1}{2}\E\paren{ \ln(1+\exp(-Yw)) \mid X=1 } + \frac{1}{2}\E\paren{ \ln(1+\exp(Yw)) \mid X=-1 } \\
		&= \frac{1}{2} \paren{ p\ln(1+\exp(-w)) + (1-p)\ln(1+\exp(w)) } + \frac{1}{2} \paren{ q\ln(1+\exp(-w)) + (1-q)\ln(1+\exp(w)) } 
	\end{align*}
	where $p\coloneqq \mathbb{P}(Y=1\mid X=1)$ and $q\coloneqq \mathbb{P}(Y=-1\mid X=-1)$.
	Collecting terms,
	\begin{align*}
		\loss(w) &= \frac{1}{2}(p+q)\ln(1+\exp(-w)) + \frac{1}{2}(1-p+1-q)\ln(1+\exp(w)) \\
		&= \frac{1}{2}(p+q)\ln\paren{\frac{1+\exp(-w)}{1+\exp(w)}}  + \ln(1+\exp(w)) \\
		&= \ln(1+\exp(w)) - \frac{p+q}{2}w.
	\end{align*}
	The assumption that $\loss$ has a unique finite minimizer 
	implies $0<p+q<2$. Define 
	\begin{align*}
		c \coloneqq \frac{p+q}{2-p-q} > 0 \implies \frac{c}{c+1} = \frac{p+q}{2},
	\end{align*}
	and so 
	\begin{align*}
		\loss(w) = \ln(1+\exp(w)) - \frac{c}{c+1}w.
	\end{align*}
	This is the one-dimensional spherical objective in 
	\cref{thm:1d-convergence-sphere}, obtained by assuming 
	$x_i = \pm 1$ and $y_i=1$, with $c$ being the ratio between 
	the number of positive and negative examples. 
	Therefore the GD dynamics under the two objectives coincide.  
\end{proof}

Before diving into the full proof of \cref{thm:pop-obj-convergence}
for general $d>1$, we first present the key lemmas used.
Note that if the signal strength $\beta = 0$ (and hence $w^*=0$), 
in which case $L=\lambda^*$ and global convergence is guaranteed 
for all step sizes below $2/L$. Therefore we may assume $\beta>0$.

\begin{lem}[Weighted decay of orthogonal component]
	\label{lemma:orthogonal-direction-decay}
	Define the weighted orthogonal error
	\begin{align}
		B_\theta(w) \coloneqq \norm{z} \exp(\theta(\beta-x)_+)
	\end{align}
	for some $\theta$. Then there exists $\theta>0$ 
	that depends only on $\beta$, $\gamma$, 
	and the dimension $d$ such that for all $z\neq 0$,
	\begin{align}
		\label{eq:change-in-B-theta}
		B_\theta(T(w)) < B_\theta(w).
	\end{align}
\end{lem}
\begin{proof}
    The proof is split into three cases based on the value of $q(r)$
    at the current $w=(x,z)\neq w^*$ (so that $B_\theta(w) > 0$). 
	Recall the decomposed GD updates 
    from \labelcref{eq:decomposed-updates-sphere-pop-obj}
    \begin{align*}
        T_x(w) &= q(r)x + \gamma\beta \\
        T_z(w) &= q(r)z
    \end{align*}
    where $r = \norm{w}$, and $q(r) \coloneq 1-\eta\lambda(r)$.

    \textbf{Case $q(r)=0$. }
	The result is immediate: since $T_z(w) = 0$ we have $B_\theta(T(w)) = 0 < B_\theta(w)$. 
    
    When $q(r)\neq 0$, we will use 
	$(\beta-x)_+ = \beta - \min\braces{x,\beta}$ to prove negativity of 
	\begin{align}
		\label{eq:log-change-in-B-theta}
		\ln\paren{\frac{B_\theta(T(w))}{B_\theta(w)}} = \ln\abs{q(r)} - \theta \underbrace{ \paren{\min\braces{T_x(w),\beta}-\min\braces{x,\beta}} }_{\eqqcolon\Delta}
	\end{align}
	provided $\theta>0$. 

	\textbf{Case $q(r)>0$. } By definition $q(r)$ and nonnegativity of 
    $\lambda$ we have $0<q(r)<1$, implying $\ln\abs{q(r)}<0$. 
    Then it suffices to show that 
	$\Delta \geq 0$. If $x\leq 0$, 
	then $T_x(w) > x$ and $\Delta > 0$ is immediate. 
	If $0 < x < \beta$, observe that 
    \begin{align*}
        T_x(w)-x &= (q(r)-1)x+\eta\lambda(\beta)\beta \tag{$\gamma = \eta \lambda(\beta)$}\\
        &= \eta\paren{\lambda(\beta)\beta-\lambda(r)x} \\
		&= \eta( \phi'(\beta) - \lambda(r)x ).
    \end{align*}
	As $\phi'$ is strictly increasing and $\lambda$ is decreasing on 
	$r > 0$ \labelcref{eq:eigval-order} and $r \geq x$, we have
	\begin{align*}
		\phi'(\beta) > \phi'(x) = \lambda(x)x \geq \lambda(r) x,
	\end{align*}
	This implies $T_x(w) > x$, again giving us $\Delta>0$. 
	Finally, if $x \geq \beta$, then $r\geq x \geq \beta$, and 
	\begin{align*}
		T_x(w) - \beta &\geq q(r)\beta - \beta + \gamma\beta \\
		&=\paren{1 - \gamma \frac{\lambda(r)}{\lambda(\beta)}}\beta + (\gamma-1)\beta \\
		&=\gamma\beta\paren{1-\frac{\lambda(r)}{\lambda(\beta)}} \geq 0.
	\end{align*}
	Thus we have $\Delta \geq 0$ when $q(r) \geq 0$ 
    and the desired result holds for arbitrary $\theta > 0$.

	\textbf{Case $q(r)<0$. } This is the most delicate case where 
    we need to analyze when $\theta$ needs to be chosen with care.
	For a fixed $r$, the quantity $\Delta$ as a function of $x$
	\begin{align*}
		x\longmapsto
		\min\braces{q(r)x+\gamma\beta,\beta}
		-\min\braces{x,\beta}
	\end{align*}
	is non-increasing as it's the sum of two non-increasing functions. 
    Since $x\leq r$, its minimum over
	$[-r,r]$ is attained at $x=r$. Define
	\begin{align*}
		F(r)\coloneqq q(r)r+\gamma\beta.
	\end{align*}
	Then \labelcref{eq:log-change-in-B-theta} becomes
	\begin{align}
		\label{eq:log-change-in-B-theta-negative-qr}
		\ln\paren{\frac{B_\theta(T(w))}{B_\theta(w)}}
		&\leq \ln\abs{q(r)}
		-\theta\paren{
			\min\braces{F(r),\beta}
			-\min\braces{r,\beta}}.
	\end{align}
	We will choose $\theta>0$ so that this upper bound is
	negative for every $r$ for which $q(r)<0$. Note that
	\begin{align*}
		F(\beta)&=\beta, \\
		F'(r)&=1-\eta\phi''(r), \\
		F(r)-r&=\eta\paren{\phi'(\beta)-\phi'(r)}.
	\end{align*}
	First suppose $\eta\phi''(\beta)\leq 1$.
	Since $\phi''$ is decreasing, $F'(r)\geq 0$ for
	$r\geq\beta$, and hence $F(r)\geq\beta$ there.
	Thus the quantity multiplying $\theta$ in
	\labelcref{eq:log-change-in-B-theta-negative-qr}
	vanishes, while
	\begin{align*}
		0<\abs{q(r)}
		=\eta\lambda(r)-1
		\leq\gamma-1<1.
	\end{align*}
	Consequently, every $\theta>0$ works for $r\geq\beta$.
	
	For $r<\beta$, strict monotonicity of $\phi'$ gives
	$F(r)>r$, so $\min\braces{F(r),\beta}-r>0$.
	If $\abs{q(r)}\leq 1$, the upper bound in
	\labelcref{eq:log-change-in-B-theta-negative-qr}
	is therefore negative for every $\theta>0$.
	To handle the remaining possibilities, i.e. 
    $r < \beta$ and $q(r) < -1$, define
	\begin{align}
		\label{eq:theta-lower-bound}
		\underline{\theta}
		\coloneqq
		\max\braces{
			0,\,
			\sup_{\substack{0\leq r<\beta\\q(r)<-1}}
			\frac{\ln\abs{q(r)}}
			{\min\braces{F(r),\beta}-r}
		},
	\end{align}
	with the convention $\sup\emptyset=-\infty$.
	Since $q(\beta)=1-\gamma>-1$, the closure of the
	set in the supremum is a compact subset of $[0,\beta)$.
	The denominator is continuous and strictly positive
	on that closure, so $\underline{\theta}<\infty$.
	Choosing $\theta=\underline{\theta}+1$ makes
	\labelcref{eq:log-change-in-B-theta-negative-qr}
	negative for every remaining $r$.

	Now suppose $\eta\phi''(\beta)>1$.
	By \labelcref{eq:eigval-order},
	\begin{align*}
		1<\eta\phi''(\beta)<\eta\lambda(\beta)=\gamma<2.
	\end{align*}
	Let $r_\eta>\beta$ be the unique $r$ satisfying
	$\phi''(r_\eta)=1/\eta$.
	On $[0,r_\eta]$, we have $q(r)<0$.
	By \cref{lemma:theta-0-curvature}, the function
	\begin{align*}
		-\frac{d}{dr}\ln\abs{q(r)}
		=-\frac{\lambda'(r)}{\lambda(r)-1/\eta}
	\end{align*}
	is strictly increasing on $(0,r_\eta)$.
	Equivalently, $\ln\abs{q(r)}$ is strictly concave there, and
	we choose $\theta$ to be
	\begin{align*}
		\theta
		\coloneqq
		-\left.\frac{d}{dr}\ln\abs{q(r)}\right|_{r=\beta}
		=-\frac{\lambda'(\beta)}{\lambda(\beta)-1/\eta}
		>0.
	\end{align*}
	Concavity along with continuity at the endpoints gives
	\begin{align}
		\label{eq:log-qr-tangent}
		\ln\abs{q(r)} & \leq \ln\abs{q(\beta)} + \frac{\lambda'(\beta)}{\lambda(\beta)-1/\eta} (r-\beta) \nonumber \\
		&= \ln(\gamma-1)-\theta(r-\beta), \qquad 0\leq r\leq r_\eta.
	\end{align}
	Now if $r<\beta$, then
	\begin{align*}
		F(r)-\beta
		&=\int_r^\beta\paren{\eta\phi''(s)-1}\,ds
		>0.
	\end{align*}
	Therefore $\min\braces{F(r),\beta}-r=\beta-r$.
	Combining
	\labelcref{eq:log-change-in-B-theta-negative-qr}
	and \labelcref{eq:log-qr-tangent} gives
	\begin{align*}
		\ln\paren{\frac{B_\theta(T(w))}{B_\theta(w)}} &\leq \ln(\gamma-1)-\theta(r-\beta)-\theta(\beta-r) \\
		&=\ln(\gamma-1)<0.
	\end{align*}

	If $\beta\leq r\leq r_\eta$, then $F'(r)\leq 0$,
	so $F(r)\leq\beta$. Moreover,
	\begin{align*}
		r+F(r)-2\beta &=\int_\beta^r\paren{2-\eta\phi''(s)}\,ds\geq 0,
	\end{align*}
	because $\eta\phi''(s)\leq\eta\phi''(\beta)<2$
	for $s\geq\beta$.
	Using \labelcref{eq:log-qr-tangent} again,
	\begin{align*}
		\ln\paren{\frac{B_\theta(T(w))}{B_\theta(w)}}&\leq \ln(\gamma-1)-\theta(r-\beta)-\theta(F(r)-\beta) \\
		&=\ln(\gamma-1)-\theta(r+F(r)-2\beta) \\
		&\leq\ln(\gamma-1)<0.
	\end{align*}

	We are left with the case $r\geq r_\eta$ while $q(r)<0$.
	On this interval, $\ln\abs{q(r)}$ is decreasing
	and $F(r)$ is non-decreasing.
	Since $r_\eta>\beta$, the RHS of
	\labelcref{eq:log-change-in-B-theta-negative-qr}
	is therefore non-increasing in $r$.
	Its value at $r_\eta$ is at most $\ln(\gamma-1)$
	by the preceding argument, so the same negative
	upper bound holds throughout this interval.

	In both regimes under the case $q(r) < 0$, 
	the required $\theta>0$ depends only on $d$, $\beta$, and $\gamma$.
	This proves \labelcref{eq:change-in-B-theta} 
    and thus the desired decrease in $B_\theta$.
\end{proof}

\begin{lem}[Signal component decay]
	\label{lemma:signal-direction-decay}
	Let 
	\begin{align*}
		f(x) = \loss((x,0)), \quad g(x) = f'(x), \quad \text{and} \quad \tau(x) = x-\eta g(x).
	\end{align*}
	On every compact interval, there exists $\rho_x < 1$ such that on that interval,
	\begin{align*}
		\abs{g(\tau(x))} \leq \rho_x \abs{g(x)}.
	\end{align*}
\end{lem}
\begin{proof}
Define
\begin{align*}
	G(x) \coloneqq g(x)+\phi'(\beta) =\sign(x)\phi'(|x|).
\end{align*}
By the properties of $\phi$, $\phi'$ is strictly increasing 
and strictly concave on $(0,\infty)$. Hence $G$ is odd 
and strictly increasing, and is concave on $(0,\infty)$.
As a result, $G\inverse$ is odd and convex on $(0,\infty)$,
with $G\inverse(0) = 0$. In particular, $G\inverse(u)/u$ 
is non-decreasing for $u > 0$.

Fix $q>0$ such that both $q$ and $-q$ belong to the range of $g$.
The corresponding points are
\begin{align*}
	g\inverse(q) = G\inverse(\phi'(\beta) + q), \quad\text{and}\quad g\inverse(-q) = G\inverse(\phi'(\beta) - q).
\end{align*}
If $q\leq \phi'(\beta)$, monotonicity of $G\inverse(u)/u$ gives 
\begin{align*}
	G\inverse(\phi'(\beta)+q) - G\inverse(\phi'(\beta)-q)  &\geq \frac{2q}{\phi'(\beta)}G\inverse(\phi'(\beta)) \\
	&= \frac{2q\beta}{\phi'(\beta)} = \frac{2q}{\lambda^*}.
\end{align*}
If $q > \phi'(\beta)$, using the fact that $G\inverse$ is odd and convex gives us 
\begin{align*}
	G\inverse(\phi'(\beta)+q) - G\inverse(\phi'(\beta)-q) &= G\inverse(\phi'(\beta)+q) + G\inverse(q-\phi'(\beta)) \\
	&\geq 2G\inverse(q) \\
	&\geq \frac{2q}{\phi'(\beta)}G\inverse(\phi'(\beta)) = \frac{2q}{\lambda^*}.
\end{align*}
Thus the opposite-gradient separation condition of 
\cref{prop:one-dimensional-gradient-contraction} holds with $\Lambda=\lambda^*$.
Along with $\gamma<2$, we have the desired contraction on compact intervals.
\end{proof}

With the above two lemmas we are ready to prove the global
convergence result on the population objective 
for the general $d>1$ case. The $d=1$ case reduces to 
our one-dimensional result as shown in \cref{lemma:population-1d-reduction}.
\highdpopconvergence*
\begin{proof}
	\textbf{Step 1.} We first construct an invariant ball 
    and show that every GD trajectory 
	must enter in finite time. 
	Since $\lambda(r) \to 0$ as $r\to\infty$, choose some $R_0 > \beta$
	such that for $r = \norm{w} \geq R_0$, $q(r)\geq 0$. As $\phi'$ is increasing, 
	we can set $\delta \coloneqq \phi'(R_0) - \phi'(\beta) > 0$, and the 
	update gives 
	\begin{align*}
		\norm{T(w)} \leq q(r)r + \gamma \beta = r - \eta(\phi'(r) - \phi'(\beta)) \leq r - \eta \delta.
	\end{align*}
	Let $R \coloneqq R_0 + \eta$. If $r\leq R_0 < R$, then $\norm{T(w)}  \leq \norm{w} + \eta \leq R$
	by Lipschitzness of $\loss$ and $\norm{X}=1$. 
	If $R_0 < r \leq R$, then from the above we have 
	$\norm{T(w)} \leq r-\eta\delta < R$. 
	Thus the ball $\mathcal{B}_R = \braces{w:\, \norm{w}\leq R}$ is invariant, 
	i.e., $\norm{T(w)} \leq R$ whenever $\norm{w} \leq R$. 
	If $r > R$, the same strict decrease applies ($\eta\delta$ is constant), 
	so every trajectory enters $\mathcal{B}_R$ in finite time.

	\textbf{Step 2.} Once inside $\mathcal{B}_R$, we can fix $\theta$ from 
	\cref{lemma:orthogonal-direction-decay} and show that the weighted 
	orthogonal component error
	\begin{align}
		B_\theta(w) = \norm{z} \exp(\theta(\beta - x)_+)
	\end{align}
	contracts at a rate $\rho_z < 1$. For $w=(x,z)$ such that 
	$r = \norm{w} \leq R$ and $z\neq 0$, 
	the decomposition in \labelcref{eq:decomposed-updates-sphere-pop-obj} 
	give us
	\begin{align*}
		\frac{B_\theta(T(w))}{B_\theta(w)} &= \frac{\abs{q(r)z}}{\abs{z}} \frac{\exp( \theta(\beta - q(r)x - \gamma\beta)_+ )}{\exp(\theta(\beta-x)_+)} \\
		&= \abs{q(r)} \exp\paren{ \theta\paren{ (\beta-q(r)x-\gamma\beta)_+ - (\beta - x)_+ } }.
	\end{align*}
	By \cref{lemma:orthogonal-direction-decay} this quantity is strictly 
	less than $1$, and continuous in $r$ and $x$ on the compact set 
	$\braces{(r,x):\, 0\leq r \leq R, \, \abs{x} \leq r}$. So its maximum, 
	which we denote by $\rho_z<1$, is attained on this set. Therefore
	for all $w\in\mathcal{B}_R$, 
	\begin{align}
		\label{eq:weighted-orthogonal-error-contraction}
		B_\theta(T(w)) \leq \rho_z B_\theta(w)
	\end{align}

	\textbf{Step 3.} The next step is to show that the signal component 
	also decays linearly inside $\mathcal{B}_R$ up to $O(\normsq{z})$. 
	At current point $w=(x,z)$, let $\bar w = (x,0)$ be its projection 
	onto the signal axis. Applying \cref{lemma:signal-direction-decay} 
	on the interval $[-R,R]$, there exists $\rho_x < 1$ such that 
	\begin{align*}
		\abs{g(T_x(\bar w))} \leq \rho_x \abs{g(x)}.
	\end{align*}
	By the update rules,
	\begin{align*}
		T_x(w) &= q(r) x + \gamma\beta = x - \eta \lambda(r)x + \gamma\beta \\
		T_x(\bar w) &= q(\abs{x}) x + \gamma\beta = x - \eta\lambda(\abs{x})x + \gamma\beta
	\end{align*}
	and so 
	\begin{align*}
		T_x(w) - T_x(\bar w) = -\eta x \paren{\lambda(r) - \lambda(\abs{x})}.
	\end{align*}
	Since 
	\begin{align*}
		\lambda'(r) = \frac{1}{r}\paren{ \phi''(r) - \lambda(r) }
	\end{align*}
	is bounded on $[0,R]$, $\lambda$ is Lipschitz on this interval. And so 
	\begin{align*}
		\abs{T_x(w) - T_x(\bar w)} &= \eta\abs{x}\abs{\lambda(r) - \lambda(\abs{x})} \\
		&\leq C_0 \abs{x} (r - \abs{x}) \\
		&= C_0\abs{x}\normsq{z} \paren{\frac{r-\abs{x}}{\normsq{z}}} \\
		&= C_0\normsq{z} \frac{\abs{x}}{r+\abs{x}} \tag{$r^2 = x^2 + \normsq{z}$} \\
		&\leq C_0 \normsq{z}.
	\end{align*}
	The function $g$ is $L_g$-Lipschitz on $[-R,R]$ and
	by definition $\norm{z}\leq B_\theta(w)$, so together 
	\begin{align*}
		\abs{g(T_x(w))} &\leq \abs{g(T_x(\bar w))} + \abs{g(T_x(w)) - g(T_x(\bar w))} \\
		&\leq \rho_x \abs{g(x)} + L_g \abs{T_x(w) - T_x(\bar w)} \\
		&\leq \rho_x \abs{g(x)} + L_g C_0 \normsq{z} \\
		&\leq \rho_x \abs{g(x)} + \underbrace{L_g C_0 }_{\eqqcolon C }B_\theta(w)^2.
	\end{align*}
	\cref{eq:weighted-orthogonal-error-contraction} in the previous step 
	implies $B_\theta(T(w))^2 \leq \rho_z^2 B_\theta(w)^2$. 
	This allows us to choose $\bar\rho$ and $\vartheta$ such that 
	\begin{align*}
		\max\braces{\rho_z^2,\,\rho_x} < \bar\rho < 1 \qquad \text{and} \qquad \vartheta \geq \frac{C}{\bar\rho - \rho_z^2}
	\end{align*}
	and define the overall Lyapunov function to be 
	\begin{align}
		V(w) \coloneqq \abs{g(x)} + \vartheta B_\theta(w)^2.
	\end{align}
	Putting it all together, 
	\begin{align*}
		V(T(w)) &= \abs{g(T_x(w))} + \vartheta B_\theta(T(w))^2 \\
		&\leq \rho_x \abs{g(x)} + CB_\theta(w)^2 + \vartheta\rho_z^2 B_\theta(w)^2 \\
		&\leq \bar \rho \abs{g(x)} + (C+\vartheta \rho_z^2) B_\theta(w)^2 \\
		&\leq \bar \rho V(w).
	\end{align*}
	Moreover, $V(w)=0$ if and only if $z=0$ and $x=\beta$, i.e. $w=w^*$, 
	which proves convergence with linear rate $\bar \rho$ once inside 
	the invariant ball. This completes the proof.
\end{proof}

\subsection{Technical lemmas}
\begin{lem}[Curvature lemma]
	\label{lemma:theta-0-curvature}
	Suppose 
	\begin{align}
		1 < \eta\phi''(\beta) < \eta\lambda(\beta) = \gamma < 2
	\end{align}
	where the second inequality follows from \labelcref{eq:eigval-order}. Then 
	there exists a unique $r_\eta > \beta$ such that 
	\begin{align*}
		\phi''(r_\eta) = \frac{1}{\eta}.
	\end{align*}
	Moreover, on the interval $(0, r_\eta)$, the function
	\begin{align}
		\psi(r) \coloneq -\frac{d}{dr}\ln\abs{q(r)} = - \frac{\lambda'(r)}{\lambda(r) - 1/\eta}
	\end{align}
	is strictly increasing.
\end{lem}
\begin{proof}
	The existence and uniqueness of $r_\eta$ follows from the assumption that 
	$\phi''(\beta) > 1/\eta$ and the fact that 
	$\phi''$ is strictly decreasing and $\phi''(r)\to0$ as $r\to\infty$. 
	We now prove that $\psi$ is strictly increasing. Differentiating $\psi$ gives
	\begin{align*}
		\psi'(r) = \frac{\lambda'(r)^2 -\lambda''(r)(\lambda(r)-1/\eta)} {(\lambda(r)-1/\eta)^2},
	\end{align*}
	so it suffices to prove that the numerator is positive. 
	After scaling by $r^2/\lambda(r)^2$, the numerator becomes
	\begin{align}
		\label{eq:scaled-numerator}
		\frac{r^2}{\lambda(r)^2} \paren{\lambda'(r)^2 -\lambda''(r)\paren{\lambda(r)-1/\eta}}
		&= \paren{ -\frac{r\lambda'(r)}{\lambda(r)} }^2 - \frac{r^2\lambda''(r)}{\lambda(r)} \frac{\lambda(r)-1/\eta}{\lambda(r)}.
	\end{align}
	Define 
	\begin{align}
		\varepsilon(r) \coloneq-\frac{r\lambda'(r)}{\lambda(r)} 
		= 1-\frac{\phi''(r)}{\lambda(r)} \quad \text{and} \quad
		\varrho(r) \coloneq \frac{\lambda(r)-1/\eta}{\lambda(r)}.
	\end{align}
	Since $1/\eta<\phi''(r)<\lambda(r)$ on $(0,r_\eta)$, we have
	$0<\varepsilon(r)<\varrho(r)<1$. 
	Differentiating $\varepsilon$ and rearranging gives 
	\begin{align*}
		\frac{r^2\lambda''(r)}{\lambda(r)} = \varepsilon(r)+\varepsilon(r)^2-r\varepsilon'(r),
	\end{align*}
	which we substitute into the scaled numerator 
	\labelcref{eq:scaled-numerator} to yield
	\begin{align*}
		\varepsilon(r)^2 - (\varepsilon(r) + \varepsilon(r)^2 - r\varepsilon'(r))\varrho(r) 
		& = \varrho(r)\paren{r\varepsilon'(r)-\varepsilon(r)} + \paren{1-\varrho(r)}\varepsilon(r)^2 \\
		& \geq \varrho(r) \paren{ \varepsilon(r) - 3\varepsilon(r)^2 } + (1 - \varrho(r)) \varepsilon(r)^2 \\
		&= \varepsilon(r) \underbrace{ \paren{\varrho(r)+\varepsilon(r) -4\varrho(r)\varepsilon(r)} }_{\eqqcolon J(r) }
	\end{align*}
	where the inequality follows from applying \cref{lemma:vareps-inequality}.
	It remains to prove that
	$J$ is positive on $(0,r_\eta)$.
	As $r\to 0$, we have $\varepsilon(r) \to 0$, while
	$\varrho(r) \to (\lambda(0) - 1/\eta)/\lambda(0) > 0$.
	Thus $J(r)>0$ for all sufficiently small $r$. 
	At $r_\eta$, the fact that $\phi''(r_\eta) = 1/\eta$ 
	implies $\varrho(r_\eta) = \varepsilon(r_\eta)$. Furthermore,
	\begin{align*}
		\lambda(r_\eta) < \lambda(\beta) = \frac{\gamma}{\eta} < \frac{2}{\eta},
	\end{align*}
	and so $\varrho(r_\eta) < 1/2$. Hence 
	\begin{align*}
		J(r_\eta) = 2\varrho(r_\eta)(1-2\varrho(r_\eta)) > 0.
	\end{align*}
	Suppose for the sake of contradiction that $J(r)=0$ for some
	$r\in(0,r_\eta)$. Solving this equation for $\varrho(r)$ gives
	\begin{align*}
		\varrho(r) = \frac{\varepsilon(r)}{4\varepsilon(r)-1}.
	\end{align*}
	Together with $0<\varepsilon(r)<\varrho(r)<1$, this forces
	$1/3 < \varepsilon(r) < 1/2$, and so at a root, 
	$1-4\varrho(r) = -1/(4\varepsilon(r)-1) < 0$.
	We now show that $J'(r) < 0$ at every root of $J$,
	which would be impossible when $J$ is positive near $0$ and 
	at the right endpoint $r_\eta$, as it would need to eventually 
	have a root with $J'(r) \geq 0$ to cross back to positive. 
	
	Differentiating $\varrho$ gives us 
	\begin{align*}
		r\varrho'(r) = -\paren{1-\varrho(r)}\varepsilon(r).
	\end{align*}
	Using this identity and the bound on $r\varepsilon'(r)$
	from \cref{lemma:vareps-inequality}, namely 
	\begin{align*}
		r\varepsilon'(r) \geq 2\varepsilon(r) - 3\varepsilon(r)^2,
	\end{align*}
	we arrive at
	\begin{align*}
		rJ'(r) &= r\varrho'(r)(1-4\varepsilon(r)) + r\varepsilon(r)'(1-4\varrho(r)) 
		\leq \frac{\varepsilon(r)(2\varepsilon(r)-1)(6\varepsilon(r)+1)}{4\varepsilon(r)-1} < 0
	\end{align*}
	at every root of $J$. This completes the proof.
\end{proof}

\begin{lem}
    \label{lemma:vareps-inequality}
    In dimension $d\geq 2$, let $U$ be the first coordinate of
    $X\sim\mathrm{Uniform}(\mathbb{S}^{d-1})$. Recall the definitions
    \begin{align*}
        \phi(r) = \E\paren{\ln(1+\exp(rU))}, 
		\quad \lambda(r) = \frac{\phi'(r)}{r}, 
		\quad 
		\text{and} 
		\quad \varepsilon(r) = 1-\frac{\phi''(r)}{\lambda(r)} = -\frac{r\lambda'(r)}{\lambda(r)}.
    \end{align*}
    Then for all $r>0$, 
	\begin{align}
		\label{eq:varepsilon-prime}
		r\varepsilon'(r) = \varepsilon(r) - \frac{r^2\lambda''(r)}{\lambda(r)} + \varepsilon(r)^2 \geq  2\varepsilon(r)-3\varepsilon(r)^2. 
	\end{align}
\end{lem}
\begin{proof}
    We first rewrite $\lambda$ and $\varepsilon$ in expectation form.
    A simple argument (\cref{lemma:first-coord-density}) 
    shows that the random variable $U$ has density
    \begin{align*}
        f(u) \propto (1-u^2)^{(d-3)/2} \qquad \text{for $ -1<u<1$}.
    \end{align*}
    Using $\frac{d}{du} (1-u^2)^{(d-1)/2} = -(d-1)u(1-u^2)^{(d-3)/2} $,
	integrating by parts gives us
    \begin{align*}
        \phi'(r) &= \E \paren{ U\sigma(rU) } 
        \propto \int_{-1}^{1} u\sigma(ru)(1-u^2)^{(d-3)/2}\,du \\
        &= -\frac{1}{d-1} \int_{-1}^{1} \sigma(ru) \frac{d}{du}(1-u^2)^{(d-1)/2}\,du \\
		&= -\frac{1}{d-1}\paren{ \left.\sigma(ru)(1-u^2)^{(d-1)/2}\right|_{-1}^1 } + \frac{1}{d-1} \int_{-1}^1 r\sigma'(ru) (1-u^2)^{(d-1)/2} \,du \\
        &= \frac{r}{d-1} \int_{-1}^{1} (1-u^2)^{(d-1)/2}\sigma'(ru)\,du \\
		&= \frac{r}{d-1} \int_{-1}^{1} (1-u^2)(1-u^2)^{(d-3)/2}\sigma'(ru)\,du \\
		&\propto \frac{r}{d-1} \E\paren{(1-U^2)\sigma'(rU)},
    \end{align*}
	so in expectation form,
    \begin{align*}
        \lambda(r) = \frac{1}{d-1} \E\paren{(1-U^2)\sigma'(rU)}.
    \end{align*}
    Differentiating in this form gives us 
    \begin{align*}
        \frac{\lambda'(r)}{\lambda(r)} = \frac{\E\paren{(1-U^2)U\sigma''(rU)}}{\E\paren{(1-U^2)\sigma'(rU)}},
        \quad \text{and} \quad
		\frac{\lambda''(r)}{\lambda(r)} = \frac{\E\paren{(1-U^2)U^2\sigma'''(rU)}}{\E\paren{(1-U^2)\sigma'(rU)}}.
    \end{align*}
	As both ratios use the same positive weight, we can define for simplicity 
	the expectation under a weighted distribution, namely, for a function $h$,
    \begin{align}
        \label{eq:tilted-expectation}
        \Eweighted\paren{h(U)}
        \coloneqq \frac{ \E\paren{(1-U^2)\sigma'(rU)h(U)} }{ \E\paren{(1-U^2)\sigma'(rU)} }.
    \end{align}
	This lets us write  
	\begin{align*}
		\Eweighted\paren{U\frac{\sigma''(rU)}{\sigma'(rU)}} = \frac{\lambda'(r)}{\lambda(r)} 
		\quad \text{and} \quad 
		\Eweighted\paren{U^2\frac{\sigma'''(rU)}{\sigma'(rU)}} = \frac{\lambda''(r)}{\lambda(r)},
	\end{align*}
    where derivatives of the sigmoid satisfy
    \begin{align*}
        \frac{\sigma''(z)}{\sigma'(z)} = -\tanh(z/2),
        \quad \text{and} \quad
        \frac{\sigma'''(z)}{\sigma'(z)} = \tanh^2(z/2)-\frac{1}{2}\sech^2(z/2),
    \end{align*}
	For a fixed $r> 0$, let 
	\begin{align*}
		A \coloneqq rU\tanh(rU/2) \geq 0,
	\end{align*}
	then $\varepsilon$ and the term in the expression 
	for $\varepsilon'$ in \labelcref{eq:varepsilon-prime} 
	can be respectively written as
    \begin{align}
        \label{eq:epsilon-as-tilted-expectation}
        \varepsilon(r) =  -\frac{r\lambda'(r)}{\lambda(r)} = \Eweighted(A),
    \end{align}
    and
    \begin{align}
        \label{eq:lambda-second-derivative}
        -\frac{r^2\lambda''(r)}{\lambda(r)} = \Eweighted\paren{ \frac{r^2U^2}{2}\sech^2(rU/2)-A^2 }.
    \end{align}
    Substituting \labelcref{eq:lambda-second-derivative} into
    \labelcref{eq:varepsilon-prime} yields
    \begin{align}
        \label{eq:r-eps-prime-expanded}
        r\varepsilon'(r) = \varepsilon(r)+
        \Eweighted\paren{
            \frac{r^2U^2}{2}\sech^2(rU/2)
        } - \Eweighted(A^2)+\varepsilon(r)^2.
    \end{align}
    Next we bound the second term. 
	For a fixed $r> 0$ denote $Y\coloneqq r\abs{U}/2$, then $A = 2Y\tanh Y$, 
	\begin{align*}
		\frac{r^2U^2}{2}\sech^2(rU/2) &= 2Y^2\sech^2Y \\
		&= 2Y^2- 2Y^2\tanh^2Y \\
		&=2Y^2 - \frac{2}{3}Y^2\tanh^2Y - \frac{4}{3}Y^2\tanh^2Y \\
		&= \frac{2Y}{3}(Y(3-\tanh^2Y)) - \frac{4}{3}Y^2\tanh^2Y \\
		&\geq 2Y\tanh Y - \frac{4}{3}Y^2\tanh^2Y 
		= A-\frac{A^2}{3},
	\end{align*}
	where the inequality uses the fact that for $Y\geq 0$, 
	$Y(3-\tanh^2 Y) \geq 3\tanh Y$. Substituting this into
    \labelcref{eq:r-eps-prime-expanded} yields
    \begin{align}
        \label{eq:r-eps-prime-reduced}
        r\varepsilon'(r) &\geq \varepsilon(r) + \Eweighted\paren{A-\frac{A^2}{3}} - \Eweighted(A^2) + \varepsilon(r)^2 \nonumber\\
        &= 2\varepsilon(r) - \frac{4}{3}\Eweighted(A^2) + \varepsilon(r)^2.
    \end{align}
    To get to the desired inequality \labelcref{eq:varepsilon-prime}, 
	it remains to show that
    \begin{align}
        \label{eq:A-second-moment-goal}
        \Eweighted(A^2) \leq 3\Eweighted(A)^2.
    \end{align}
    Expanding the weighted expectation $\Eweighted$ in \labelcref{eq:tilted-expectation} 
	using the density of $U$,
    \begin{align*}
        \Eweighted(h(U)) &= \frac{ \int_{-1}^1 h(u)(1-u^2)\sigma'(ru)f(u) \,du }{ \int_{-1}^1 (1-u^2)\sigma'(ru)f(u) \,du } = \int_{-1}^1 h(u) f_{(r)}(u) \, du,
    \end{align*}
    where $f_{(r)}$ is the density of $U$ under the weighted expectation, namely
    \begin{align*}
        f_{(r)}(u) &= \frac{(1-u^2)\sigma'(ru)f(u)}{\int_{-1}^1(1-u^2)\sigma'(ru)f(u) \,du} \propto (1-u^2)^{(d-1)/2}\sigma'(ru) \qquad \text{for $ -1<u<1$}.
    \end{align*}
    Now fix $r>0$ and define the deterministic function
    \begin{align*}
        a(u) &\coloneqq ru\tanh(ru/2)
    \end{align*}
    where $A=a(U)$. Also define the random variable
	\begin{align*}
        D &\coloneqq \frac{(d-1)U^2}{1-U^2},
	\end{align*}
	and note that 
	\begin{align*}
		\Eweighted(D) &= \int_{-1}^1 \frac{(d-1)u^2}{1-u^2} f_{(r)}(u) \,du \\
		&\propto \int_{-1}^1 \frac{(d-1)u^2}{1-u^2} (1-u^2)^{(d-1)/2} \sigma'(ru) \,du \\
		&= (d-1)\int_{-1}^1 u^2 (1-u^2)^{(d-3)/2} \sigma'(ru) \,du,
	\end{align*}
	which is finite as $\sigma'$ is bounded and $d \geq 2$.
	Taking the derivative $f'_{(r)}$ and multiplying by $-u/f_{(r)}$ gives us 
    \begin{align*}
        -u\frac{f_{(r)}'(u)}{f_{(r)}(u)} &= -u \paren{ r\frac{\sigma''(ru)}{\sigma'(ru)} - \frac{(d-1)u}{1-u^2} } \\
		&= a(u)+\frac{(d-1)u^2}{1-u^2}.
    \end{align*}
	As $u\to\pm 1$, $f_{(r)}(u) \to 0$, and
	\begin{align}
		0 = \left.uf_{(r)}(u)\right|_{u=-1}^{u=1} &= \int_{-1}^1 \paren{uf_{(r)}(u)}'\,du \nonumber \\
		&= \int_{-1}^1 f_{(r)}(u)\,du + \int_{-1}^1 u f_{(r)}'(u)\,du \nonumber \\
		&= \int_{-1}^1 f_{(r)}(u)\,du - \int_{-1}^1 \paren{a(u) + \frac{(d-1)u^2}{1-u^2}} f_{(r)}(u) \,du \nonumber \\
		&= 1-\Eweighted(A+D). \label{eq:A-D-first-identity}
	\end{align}
	Similarly, 
	\begin{align*}
		0 = \left.ua(u)f_{(r)}(u)\right|_{u=-1}^{u=1} &= \int_{-1}^1 \paren{ua(u)f_{(r)}(u)}'\,du \\
		&= \int_{-1}^1 \paren{a(u)+ua'(u)}f_{(r)}(u)\,du + \int_{-1}^1 ua(u)f_{(r)}'(u)\,du  \nonumber \\
		&= \int_{-1}^1 \paren{a(u)+ua'(u)}f_{(r)}(u)\,du - \int_{-1}^1 a(u)\paren{a(u)+\frac{(d-1)u^2}{1-u^2}} f_{(r)}(u)\,du \nonumber \\
		&= \Eweighted\paren{A+Ua'(U)} - \Eweighted\paren{A(A+D)},
	\end{align*}
	which implies
	\begin{align}
		\label{eq:A-D-second-identity}
		\Eweighted\paren{A(A+D)}
		&= \Eweighted\paren{A+Ua'(U)} \nonumber\\
		&= 2\Eweighted(A) + \Eweighted(Ua'(U) - a(U)) \nonumber \\
		&= 2\Eweighted(A) + \Eweighted\paren{a(U) + \frac{r^2U^2}{2}\sech^2(rU/2) - a(U)} \nonumber \\
		&\leq 2\Eweighted(A) + \Eweighted\paren{a(U) + a(U) - a(U)} = 3\Eweighted(A).
	\end{align}
	where the inequality follows from the fact 
	that $y\sech^2y\leq\tanh y$ for $y\geq 0$.
	Now let $\hat U$ be an independent copy of $U$ that produces 
	$\hat A$ and $\hat D$, that is, $\hat A$ and $\hat D$ are 
	independent copies of $A$ and $D$, respectively under the same $\hat U$.
	Since $A$ and $D/A$ are both increasing in $|U|$, the differences 
	$A-\hat A$ and $D/A - \hat D/\hat A$ have the same sign. Multiplying them 
	by $A\hat A\geq 0$ results in the inequality 
	\begin{align*}
		0 \leq A\hat A(A-\hat A)\paren{\frac{D}{A} - \frac{\hat D}{\hat A}} = (A-\hat A)(\hat A D-A\hat D).
	\end{align*}
	Taking expectation and using independence gives us 
	\begin{align*}
		0 \leq 2\paren{\Eweighted(A)\Eweighted(AD) - \Eweighted(A^2)\Eweighted(D)} 
		\implies \Eweighted(AD) \geq \frac{\Eweighted(A^2) \Eweighted(D)}{\Eweighted(A)}.
	\end{align*}
	Combining this and \labelcref{eq:A-D-second-identity},
	\begin{align*}
		3\Eweighted(A) &\geq \Eweighted(A^2 + AD) \\
		&\geq \Eweighted(A^2) + \frac{\Eweighted(A^2) \Eweighted(D)}{\Eweighted(A)} \\
		&= \frac{\Eweighted(A^2) \paren{\Eweighted(A+D)} }{\Eweighted(A)} 
		= \frac{\Eweighted(A^2)}{\Eweighted(A)} \tag{Using \labelcref{eq:A-D-first-identity}}.
	\end{align*}
    Lastly, combining this with
    \labelcref{eq:epsilon-as-tilted-expectation,eq:r-eps-prime-reduced}
    gives
    \begin{align*}
        r\varepsilon'(r) &\geq 2\varepsilon(r) -\frac{4}{3}\paren{3\varepsilon(r)^2} + \varepsilon(r)^2 \\
        &=2\varepsilon(r)-3\varepsilon(r)^2,
    \end{align*}
    and the proof is complete.
\end{proof}

\begin{lem}
    \label{lemma:first-coord-density}
    Let $U$ be the first coordinate of $X\sim\mathrm{Uniform}(\mathbb{S}^{d-1})$.
    Then it has probability density function 
    \begin{align*}
        f(u) \propto (1-u^2)^{(d-3)/2}, \qquad -1 < u < 1.
    \end{align*}
\end{lem}
\begin{proof}
    Let $Z\in\R^d$ be a random variable where each coordinate is 
    sampled independently from $\gN(0,1)$. Then $X \overset{d}{=} Z/\norm{Z}$,
    and so 
    \begin{align*}
        U^2 &= \frac{Z_1^2}{Z_1^2 + \sum_{j=2}^d Z_j^2}.
    \end{align*}
    Since $Z_1^2\sim\chi_1^2 = \mathrm{Gamma}\paren{\frac{1}{2},2}$ and 
    $\sum_{j=2}^dZ_j^2\sim\chi_{d-1}^2 = \mathrm{Gamma}\paren{\frac{d-1}{2},2}$
    and they are independent, we have 
    $U^2\sim\mathrm{Beta}\paren{\frac{1}{2},\frac{d-1}{2}}$ which proves 
    the result up to normalization.
\end{proof}

%% file: sections/appendix/app-1d-suff-cond.tex
\section{Proof under an annulus condition in one dimension}
\label{sec:app-1d-suff-cond-proof}
In this section, we prove that in one dimension, 
relaxing the spherical data assumption 
to an annulus condition still guarantees global convergence for 
all step sizes up to the local stability threshold.
\onedimsufficientcondition*
\begin{proof}
	Let $g \coloneqq \loss'$ and $h \coloneqq \loss''$.
	The annulus condition $\nu\leq 4$ lets us use
	\cref{lemma:annulus-phi-convexity} to show that the function
	\begin{align*}
		\phi(s)\coloneqq (g\inverse)'(s)
	\end{align*}
	is convex on the range of $g$. Since $\eta=\gamma/\lambda^*$ and
	$\gamma\in(0,2)$, we have $0 < \eta < 2/g'(w^*)$.
	Therefore we can apply \cref{cor:one-dimensional-gradient-contraction} 
	with $F=T$ which gives
	\begin{align*}
		\abs{g(T(w))}<\abs{g(w)}
	\end{align*}
	for all $w\neq w^*$. Moreover, $\abs{g(w_t)}\to 0$ for every
	initialization. Since $g$ is strictly increasing with unique root at
	$w^*$, it follows that $w_t\to w^*$.
	Now consider $\gamma\in(1,2)$ and let 
	\begin{align*}
		\alpha_\gamma \coloneqq \frac{2-\gamma}{2+\gamma} \quad \text{and} \quad \gN_\gamma \coloneqq \braces{w:\, \abs{g(w) \leq \alpha_\gamma \lambda^*}}.
	\end{align*}
	Since $\alpha_\gamma\lambda^*>0$ and
	$\abs{g(w_t)}\to 0$, every trajectory enters $\gN_\gamma$ in a finite
	number of iterations. The strict decrease of $\abs{g(w_t)}$ also shows
	that $\gN_\gamma$ is an invariant set.
	Under our assumption 
	$\norm{x_i} \leq 1$, the logistic regression objective satisfies the 
	generalized self-concordance property with constant $1$ 
	\citep{bach2010self}, and hence 
	\begin{align*}
		\abs{h'(w)} \leq h(w) \qquad \forall w\in \R.
	\end{align*}
	Integrating $\abs{ (\ln h)' } \leq 1$ gives for all $u,v\in \R$,
	\begin{align}
		\exp(-\abs{u-v}) \leq \frac{h(u)}{h(v)} \leq \exp( \abs{u-v} ).
	\end{align}
	Suppose $\abs{g(w)} \leq \alpha \lambda^*$ for some $\alpha\in(0,1)$, 
	and let $\delta\coloneqq \abs{w-w^*}$. Then 
	\begin{align*}
		\abs{g(w)} = \abs{ \int_{w^*}^w h(s) \,d s } 
		\geq \lambda^* \int_0^\delta \exp(-s) \,d s = \lambda^* (1-\exp(-\delta)).
	\end{align*}
	Hence $\delta \leq -\ln(1-\alpha)$, and applying 
	the above inequality again yields 
	\begin{align}
		\label{eq:sandwich-h-over-lambda}
		1-\alpha \leq \frac{h(w)}{\lambda^*} \leq \frac{1}{1-\alpha}.
	\end{align}
	Let $w_t\in\gN_\gamma$, so $w_{t+1} \in \gN_{\gamma}$. Since $g$ 
	is increasing, the entire interval between $w_t$ and $w_{t+1}$ is 
	contained in $\gN_{\gamma}$. The mean value theorem guarantees a
	$z_t$ on this interval such that 
	\begin{align*}
		g(w_{t+1}) = (1-\eta h(z_t)) g(w_t).
	\end{align*}
	Applying \labelcref{eq:sandwich-h-over-lambda} with $\alpha = \alpha_\gamma$ 
	yields 
	\begin{align*}
		\frac{2\gamma}{ 2+\gamma } \leq \frac{h(z_t)}{\lambda^*} \leq \frac{2+\gamma}{2\gamma}.
	\end{align*}
	As $\abs{1-\gamma x}$ is convex in $x$, its maximum on this 
	interval is at an endpoint. At the upper endpoint it equals $\gamma/2$,
	while at the lower endpoint we have for $\gamma\in(1,2)$, 
	\begin{align*}
		\abs{ 1 - \frac{2\gamma^2}{2 + \gamma} } \leq \frac{\gamma}{2}.
	\end{align*}
	Thus we have for all $w_t\in\gN_\gamma$,
	\begin{align*}
		\abs{g(w_{t+1})} \leq \frac{\gamma}{2} \abs{g(w_t)},
	\end{align*}
	and the proof is complete.
\end{proof}

\begin{lem}
	\label{lemma:annulus-phi-convexity}
	If $\nu \leq 4$, then for all $w\in\R$,
	\begin{align}
		\label{eq:hessian-differential-inequality}
		3h'(w)^2 - h(w)h''(w) > 0.
	\end{align}
	Consequently, 
	\begin{align}
		\label{eq:phi-inverse-hessian-displacement}
		\phi(s) \coloneqq (g\inverse)'(s) = \frac{1}{h(g\inverse(s))}
	\end{align}
	is strictly convex on the range of $g$.
\end{lem}
\begin{proof}
	The Hessian of the logistic regression objective can be written as 
	\begin{align*}
		h(w) &= \frac{1}{n} \sum_{i=1}^n x_i^2 \sigma'(x_i w) = \frac{1}{4n} \sum_{i=1}^n r_i^2 \sech^2\paren{ \frac{r_i w}{2} }.
	\end{align*}
	At $w=0$, $h'(0)=0$ and $h''(0)<0$, so the inequality 
	\labelcref{eq:hessian-differential-inequality} holds.
	As $h$ is even, so is the LHS of \labelcref{eq:hessian-differential-inequality},
	therefore it suffices to prove the inequality for $w>0$. 
	The derivatives of $h$ are 
	\begin{align*}
		h'(w)  &= -\frac{1}{4n}\sum_{i=1}^n r_i^3 \sech^2 \paren{ \frac{r_iw}{2} } \tanh \paren{ \frac{r_iw}{2} }, \\
		h''(w) &= \frac{1}{8n} \sum_{i=1}^n r_i^4 \sech^2 \paren{ \frac{r_iw}{2} } \paren{ 2\tanh^2 \paren{ \frac{r_iw}{2} }  - \sech^2 \paren{ \frac{r_iw}{2} } } \\
		&= \frac{1}{4n} \sum_{i=1}^n r_i^4 \sech^2 \paren{ \frac{r_iw}{2} } \paren{ 1-\sech^2 \paren{ \frac{r_iw}{2} } - \frac{1}{2} \sech^2 \paren{ \frac{r_iw}{2} }  } \\
		&= \frac{1}{4n} \sum_{i=1}^n r_i^4 \sech^2 \paren{ \frac{r_iw}{2} } \paren{ 1- \frac{3}{2} \sech^2 \paren{ \frac{r_iw}{2} }  }.
	\end{align*}
	Since $h(w)>0$, after dividing the desired inequality 
	\labelcref{eq:hessian-differential-inequality} by $h(w)^2$, 
	it suffices to prove
	\begin{align}
		\label{eq:normalized-hessian-differential-inequality}
		\frac{h''(w)}{h(w)} < 3\paren{\frac{h'(w)}{h(w)}}^2.
	\end{align}
	To this end, for fixed $w>0$ define the weights 
	\begin{align}
		\alpha_i \coloneqq \frac{r_i^2\sech^2(r_iw/2)}{\sum_j r_j^2\sech^2(r_jw/2)},
		\qquad z_i\coloneqq r_i\tanh(r_iw/2).
	\end{align}
	Then $\sum_i\alpha_i=1$, which allows us to express the ratios 
	as the following weighted sums:
	\begin{align*}
		-\frac{h'(w)}{h(w)} &=\sum_i\alpha_i z_i, \\ 
		\frac{h''(w)}{h(w)} &=\sum_i\alpha_i r_i^2 \paren{1-\frac32\sech^2(r_iw/2)} \\
		&=\sum_i \alpha_i r_i^2\paren{ 1 - \frac{3}{2} \paren{ 1 - \tanh^2(r_iw/2)} } \\
		&=\frac32\sum_i\alpha_i z_i^2 -\frac12\sum_i\alpha_i r_i^2. 
	\end{align*}
	Let $p\coloneqq\tanh(w/2)$. Since $r_i\leq1$ and $\tanh$ is
	increasing,
	\begin{align*}
		z_i=r_i\tanh(r_iw/2)\leq r_ip.
	\end{align*}
	Consequently,
	\begin{align}
		\label{eq:h-double-prime-second-moment-bound}
		\frac{h''(w)}{h(w)} &\leq \paren{\frac32-\frac{1}{2p^2}} \sum_i\alpha_i z_i^2
		= \frac{3p^2-1}{2p^2} \sum_i\alpha_i z_i^2.
	\end{align}
	If $3p^2-1\leq0$ then \labelcref{eq:h-double-prime-second-moment-bound}
	is negative, and
	\labelcref{eq:normalized-hessian-differential-inequality}
	follows immediately. We therefore consider $3p^2-1>0$. 
	To control the second moment in
	\labelcref{eq:h-double-prime-second-moment-bound}, observe that
	in each $z_i$, the function
	$r\mapsto r\tanh(rw/2)$ is increasing on $r>0$. Since $\nu\leq4$,
	after rescaling we have $r_i\in[1/4,1]$, and hence
	\begin{align*}
		\frac14\tanh(w/8)\leq z_i\leq\tanh(w/2).
	\end{align*}
	Writing $z_{\min}\coloneqq\min_i z_i$ and $z_{\max}\coloneqq\max_i z_i$, 
	this gives
	\begin{align}
		\frac{z_{\max}}{z_{\min}} \leq \bar z(w) \coloneqq 4\frac{\tanh(w/2)}{\tanh(w/8)}.
	\end{align}
	Using $\beta_i\coloneqq\alpha_i z_i/\sum_j\alpha_jz_j$ as weights,
	the Kantorovich inequality gives
	\begin{align*}
		\frac{\sum_i\alpha_i z_i^2} {\paren{\sum_i\alpha_i z_i}^2}
		= \paren{\sum_i\beta_i z_i} \paren{\sum_i\frac{\beta_i}{z_i}} 
		&\leq \frac{(z_{\max}+z_{\min})^2}{4z_{\max}z_{\min}} \\
		&= \frac{ (z_{\max}/z_{\min} + 1)^2 }{ 4 z_{\max}/z_{\min} }
		\leq \frac{(\bar z(w)+1)^2}{4\bar z(w)} 
	\end{align*}
	where the last inequality comes from the fact that
	$x\mapsto(x+1)^2/(4x)$ is increasing on $x \geq 1$. Altogether,
	\begin{align*}
		\frac{h''(w)}{h(w)} &\leq \frac{3p^2-1}{2p^2}
		\frac{(\bar z(w)+1)^2}{4\bar z(w)} \paren{\frac{h'(w)}{h(w)}}^2.
	\end{align*}
	It remains to prove
	\begin{align}
		\label{eq:final-scalar-inequality}
		\frac{3p^2-1}{2p^2} \frac{(\bar z(w)+1)^2}{4\bar z(w)} < 3.
	\end{align}
	to get the desired \labelcref{eq:hessian-differential-inequality}.
	To this end, let $x \coloneqq \tanh(w/8)$.
	Since $p=\tanh(w/2)=\tanh(4\arctanh x)$, we have
	\begin{align*}
		p=\frac{4x(1+x^2)}{1+6x^2+x^4},
		\quad \text{and} \quad 
		\bar z(w)=\frac{4p}{x}.
	\end{align*}
	Direct expansion of \labelcref{eq:final-scalar-inequality} in terms 
	of $x$ gives us
	\begin{align}
		\label{eq:polynomial-reduction}
		3 - \frac{3p^2-1}{2p^2} \frac{(\bar z(w)+1)^2}{4\bar z(w)}
		&= \frac{P(x)}{2048x^2(1+x^2)^3(1+6x^2+x^4)},
	\end{align}
	where
	\begin{align}
		\label{eq:polynomial-P}
		P(x)=
		x^{16}
		+8x^{14}
		+5020x^{12}
		+34808x^{10}
		+76902x^8
		+62776x^6
		+12124x^4
		-3512x^2
		+289.
	\end{align}
	The denominator in \labelcref{eq:polynomial-reduction} is strictly
	positive for all $x\neq 0$. Every term of the numerator $P(x)$
	above degree $4$ is non-negative, and the remaining terms 
	can be completed as 
	\begin{align*}
		12124x^4 - 3512x^2 + 289 = (110x^2 - 16)^2 + 24x^4 + 8x^2 + 33
	\end{align*}
	which is also positive. Thus we have 
	proven \labelcref{eq:final-scalar-inequality} holds, and hence
	\begin{align*}
		3h'(w)^2-h(w)h''(w)>0
	\end{align*}
	for all $w>0$. Finally, taking the derivative of $\phi$ in 
	\labelcref{eq:phi-inverse-hessian-displacement} gives us 
	\begin{align*}
		\phi''(s) = \frac{3h'(w)^2 - h(w)h''(w) }{ h(w)^5 } > 0
	\end{align*}
	using positivity of $h$ and the inequality we just proved,
	hence strict convexity of $\phi$, and the proof is complete.
\end{proof}

%% file: sections/appendix/app-misc.tex
\section{Miscellaneous results}
\label{sec:app-misc}

The following proof of the two-dimensional counterexample on the sphere 
(presented in \cref{sec:two-dim-counterexample})
relies heavily on certificates from a Computer Algebra System (CAS).
Specifically, we use \texttt{SymPy} \citep{meurer2017sympy} to certify 
quantities of interest (such as the unique minimizer, eigenvalues, etc.) 
are within certain range, in order to prove desired identities 
and inequalities. We highlight usages of \blue{CAS} in blue, and the 
corresponding Python script can be found at 
\url{https://gist.github.com/mengxixi/040e2e9b98ae2548cb8052c68ed5dc99}.
\twodimsphericalcounterexample*
\begin{proof}
	It's easy to see that all examples lie on the unit sphere,
	and that the problem admits a unique minimizer.
	Let $F\coloneqq n\loss$ and $\ell(\cdot) = \ln(1+\exp(\cdot))$
	be the logistic loss, then for $w=(x,y)$,
	\begin{align*}
		F(w) = 800\ell(-x) + \ell(x) + 4\ell(-ax-by) + 4\ell(-ax+by),
	\end{align*}
	and so 
	\begin{align*}
		\nabla_x F(w) &= 801\sigma(x) - 800 + 4a(\sigma(ax + by) + \sigma(ax-by) - 2) \\
		\nabla_y F(w) &= 4b(\sigma(ax+by) - \sigma(ax-by)).
	\end{align*}
	As the dataset is symmetric about the $x$-axis, the minimizer is at
	$w^*=(x^*,0)$, where $x^*$ is the unique root of
	\begin{align}
		0 = 801\sigma(x^*) - 800 + 8a(\sigma(ax^*)-1).
	\end{align}
	The off-diagonal entries of the Hessian are given by 
	\begin{align*}
		\nabla^2_{xy} F(w) = \nabla^2_{yx} F(w) &= 4ab(\sigma'(ax+by) - \sigma'(ax-by))
	\end{align*}
	which evaluates to $0$ at $w^*$ as $y^*=0$. The diagonal entries are 
	\begin{align*}
		\nabla^2_{xx} F(w) &= 4a^2(\sigma'(ax+by) + \sigma'(ax-by)) + 801\sigma'(x) \\
		\nabla^2_{yy} F(w) &= 4b^2(\sigma'(ax+by) + \sigma'(ax-by)).
	\end{align*}
	Evaluating at $w^*$, 
	\begin{align*}
		n \lambda_1 \coloneqq \nabla^2_{xx} F(w) \mid_{w=w^*} &= 8a^2\sigma'(ax^*) + 801\sigma'(x^*) \\
		n \lambda_2 \coloneqq \nabla^2_{yy} F(w) \mid_{w=w^*} &= 8b^2\sigma'(ax^*),
	\end{align*}
	which are also the two eigenvalues. \blue{CAS1} certifies that 
	\begin{align}
		\lambda^* = \frac{8b^2 \sigma'(ax^*)}{809} = \lambda_2 > \lambda_1.
	\end{align}
	Let $v = 12/5$ and $u$ be the unique root of 
	\begin{align*}
		\nabla_x F(u,v) = \underbrace{801\sigma(u) - 800 + 4a(\sigma(au+bv) + \sigma(au-bv)-2) }_{\eqqcolon g(u)}= 0.
	\end{align*}
	Indeed, observe that $g$ is continuous and strictly increasing as $g'>0$. Moreover, 
	\begin{align*}
		\lim_{z\to-\infty} g(z) = -800-8a < 0 \quad \text{and} \quad \lim_{z\to +\infty} g(z) = 1 > 0.
	\end{align*}
	Therefore $g$ has a unique and finite root $u$. Let 
	\begin{align}
		w_1 = (u,v)\transpose \quad \text{and} \quad w_2 = (u,-v)\transpose.
	\end{align}
	We will show that for some $\gamma<2$, these two points make up a stable $2$-cycle.
	At these two points, the first coordinates of the gradients are both $0$,
	while the second coordinates are 
	\begin{align*}
		\nabla F_y(u,v) \quad \text{and} \quad \nabla F_y(u,-v) = -\nabla F_y(u,v).
	\end{align*}
	So regardless of the step size the first coordinate will remain unchanged.
	For the second coordinate, choosing $\eta = n(2v)/\nabla F_y(u,v)$ gives us 
	\begin{align*}
		v - \frac{2v}{\nabla F_y(u,v)} \nabla F_y(u,v) = -v \quad \text{and} \quad -v - \frac{2v}{\nabla F_y(u,v)} \nabla F_y(u,-v) = v
	\end{align*}
	which is a $2$-cycle. Now to need to compute the $\gamma$ corresponding 
	to this step size:
	\begin{align}
		\gamma = \eta\lambda^* = \frac{809(2v)}{\nabla F_y(u,v)} \frac{8b^2\sigma'(ax^*)}{809} = \frac{ 16vb^2\sigma'(ax^*) }{\nabla F_y(u,v)}.
	\end{align}
	\blue{CAS2} certifies that this $\gamma<2$, and numerically 
	\begin{align}
		\gamma \approx 1.882561517573605646.
	\end{align}
	To show that the cycle is stable, we will show that the Jacobian 
	of the $2$-step map at its fixed point $w_1$ (or $w_2$)
	has eigenvalue strictly within the unit circle:
	\begin{align*}
		J_wT^2(w)\mid_{w=w_1} &= (I - \eta\nabla^2\loss(w_2))(I - \eta\nabla\loss^2(w_1)) \\
		&= \paren{ I - \frac{2v}{ \nabla F_y(u,v) } \nabla^2 F(w_2) }\paren{ I - \frac{2v}{ \nabla F_y(u,v) } \nabla^2 F(w_1) }.
	\end{align*}
	Let $\tau$ and $\delta$ denote the trace and determinant 
	of this Jacobian, respectively. \blue{CAS3} certifies
	\begin{align*}
		0.880710<\tau<0.880862, \qquad 0.059220<\delta<0.059322.
	\end{align*}
	In particular,
	$$\tau^2-4\delta>(0.880710)^2-4(0.059322)>0,$$
	so the two eigenvalues are real. Since their product is $\delta>0$ and 
	their sum is $\tau>0$, both eigenvalues are positive. 
	Moreover, their sum satisfies $\tau<1$, so both eigenvalues 
	lie strictly inside the unit circle. Numerically, the two eigenvalues are
	$$0.07342575 \qquad \text{and} \qquad 0.80735246.$$
	Therefore, the $2$-cycle is locally asymptotically stable, and convergence 
	to this cycle is guaranteed when initialized in its neighbourhood.
\end{proof}

The following is a straightforward argument showing that a GD trajectory 
using a step size of the form \labelcref{eq:step-size} 
is invariant under uniform scaling of the dataset.
\begin{restatable}{fact}{gdpreservedunderscaling}
	\label{prop:gd-preserved-under-scaling}
	Let $X$ be the original dataset and all 
	labels be $1$. Denote the corresponding logistic regression objective 
	by $\loss$, its solution by $w^*$, and 
	$\lambda \coloneqq \lambda_{\max} ( \nabla^2\loss(w^*) )$. 
	Let $\braces{w_t}$ be any trajectory 
	generated by GD on $\loss$ using the step size $\eta = \gamma / \lambda$
	for some $\gamma>0$. 

	Now suppose we scale the dataset by some $c>0$, giving us 
	$\hat X = cX$. The corresponding objective is denoted by $\hat \loss$, 
	its solution by $\hat w^*$, 
	and $\hat\lambda \coloneqq \lambda_{\max} ( \nabla^2 \hat\loss(\hat w^*))$.
	Then, $\braces{\hat w_t} = \braces{w_t / c}$ is a valid trajectory 
	that can be generated by GD on $\hat \loss$ with the 
	step size $\hat\eta = \gamma / \hat \lambda$.
\end{restatable}
\begin{proof}
	The gradient of the scaled objective is given by 
	\begin{align*}
		\nabla \hat \loss(w) &= -\frac{1}{n} \hat X\transpose \sigma (-\hat X w) \\
		&= -c\frac{1}{n}X\transpose\sigma(-cXw) = c\nabla \loss(cw).
	\end{align*}
	Since $\nabla \loss (w^*) = 0$, $\hat w^* = w^*/c$ is the minimizer of $\hat \loss$. 
	Similarly, the new Hessian at the new solution is given by 
	\begin{align*}
		\nabla^2 \hat\loss(\hat w^*) &= \frac{1}{n}\hat X\transpose \Diag(\sigma'(-\hat X \hat w^*)) X \\
		&= \frac{c^2}{n} X\transpose \Diag( \sigma'( -Xw^* ) ) X = c^2 \nabla^2\loss(w^*),
	\end{align*}
	and so $\hat\lambda = c^2\lambda$. Now
	take the GD update that generates the original trajectory 
	and divide it by $c$ on both sides,
	\begin{align*}
		\frac{1}{c}w_{t+1} &= \frac{1}{c} w_t - \frac{\gamma}{c\lambda} \nabla \loss(w_t) \\
		&= \frac{1}{c} w_t - \frac{\gamma}{\hat \lambda} c \nabla \loss \paren{ \frac{c}{c} w_t } \\
		&= \frac{1}{c} w_t - \frac{\gamma}{\hat \lambda} \nabla \hat \loss\paren{ \frac{1}{c} w_t }
	\end{align*}
	which completes the proof.
\end{proof}

The following two results are used throughout the paper whenever 
convergence in a one-dimensional GD map is concerned. In short, 
they show that in one dimension, as long as the function and the 
step size are such that one cannot produce a $2$-cycle, then convergence 
is guaranteed via monotonic decrease in gradient norm.
\begin{restatable}{prop}{onedimgradnormcontraction}
	\label{prop:one-dimensional-gradient-contraction}
	Let \(g:\mathbb R\to\mathbb R\) be continuously differentiable and
	strictly increasing, with unique root at $x^*$. Assume 
	$g'(x^*) > 0$ and that there exists $\Lambda > 0$ such that whenever 
	$q>0$ and that both $q$ and $-q$ are in $\mathrm{Range}(g)$,
	\begin{align}
		\label{eq:opposite-grad-separation}
		g\inverse(q) - g\inverse(-q) \geq \frac{2q}{\Lambda}.
	\end{align}
	For $0 < \eta < 2/\Lambda$, let $F(x) \coloneqq x - \eta g(x)$.
	Then for all $x\neq x^*$, $\abs{g(F(x))} < \abs{g(x)}$.
	Moreover, on every compact interval $I\subset \R$, there exists $\rho_I<1$
	such that on this interval,
	\begin{align*}
		\abs{g(F(x))} < \rho_I \abs{g(x)}.
	\end{align*}
	Therefore, for all $x_0\in\R$, the trajectory generated by 
	$x_{t+1} = F(x_t)$ starting at $x_0$ satisfies 
	\begin{align*}
		\abs{g(x_t)} \leq \rho^t \abs{g(x_0)}
	\end{align*}
	for some $\rho<1$ depending on $x_0$.
\end{restatable}
\begin{proof}
	If $x > x^*$, then $q\coloneqq g(x) > 0$ and $F(x) < x$. If $F(x) \geq x^*$,
	strict monotonicity of $g$ implies 
	\begin{align*}
		0 \leq g(F(x)) < g(x) = q
	\end{align*}
	as desired. Suppose instead $F(x) < x^*$. If $-q\in\mathrm{Range}(g)$, 
	let $x_-(q) \coloneq g\inverse(-q)$
	be the point on the other side of $x^*$ with the same gradient magnitude 
	as $x$. Then \labelcref{eq:opposite-grad-separation} implies 
	\begin{align*}
		x - x_-(q) \geq \frac{2q}{\Lambda} > \eta q.
	\end{align*}
	Therefore $F(x) > x_-(q)$, from which strict monotonicity of $g$ yields 
	\begin{align*}
		-q < g(F(x)) < q.
	\end{align*}
	If $-q\notin\mathrm{Range}(g)$, then $F(x) < x^*$ must have 
	a negative gradient with magnitude already strictly less than $q$.
	Thus $\abs{g(F(x))} < \abs{g(x)}$ whenever $x>x^*$, and the argument
	for when $x<x^*$ is identical with signs reversed.
	
	It remains to show contraction on compact intervals. For $x\neq x^*$, 
	define 
	\begin{align*}
		R(x) \coloneqq \frac{\abs{g(F(x))}}{\abs{g(x)}} < 1.
	\end{align*}
	For sufficiently small $q$, both points with opposite gradient exist, and 
	by \labelcref{eq:opposite-grad-separation},
	\begin{align*}
		\frac{1}{\Lambda} \leq \lim_{q\downarrow 0} \frac{g\inverse(q) - g\inverse(-q)}{2q} = \frac{1}{g'(x^*)}.
	\end{align*}
	Taylor expansion of $g$ at $x^*$ gives 
	\begin{align*}
		R(x) &= \frac{ \abs{g'(x^*)(x-x^*-\eta g(x)) + o(\abs{x-x^*})  } }{ \abs{g'(x^*)(x-x^*) + o(\abs{x-x^*}) }} \\
		&= \frac{ \abs{g'(x^*)(x-x^*-\eta ( g'(x^*)(x-x^*) )) + o(\abs{x-x^*})  } }{ \abs{g'(x^*)(x-x^*) + o(\abs{x-x^*}) }} \\
		&= \frac{ g'(x^*)(x-x^*)(1-\eta g'(x^*)) + o(\abs{x-x^*}) }{ \abs{g'(x^*)(x-x^*) + o(\abs{x-x^*}) }}.
	\end{align*}
	Dividing by $\abs{x-x^*}$ top and bottom and taking the limit gives us
	\begin{align*}
		\lim_{x\to x^*} R(x) = \abs{1 - \eta g'(x^*)} < 1.
	\end{align*}
	So $R$ extends continuously to $x^*$ with a value strictly less than $1$.
	Therefore its maximum on any compact interval $I$ is attained at 
	some $\rho_I<1$. Finally, let $q_0 \coloneqq \abs{g(x_0)}$. Every 
	subsequent update moves in the direction of $x^*$, 
	but if an update were to cross $x^*$, it cannot be further than $\eta q_0$
	away from $x^*$, as gradient norms and hence step lengths decrease 
	strictly at every step. By induction, the entire trajectory must 
	lie in the compact interval 
	\begin{align*}
		I_0 \coloneqq \brackets{\min\braces{x_0, \, x^*-\eta q_0}, \, \max\braces{x_0,\, x^*+\eta q_0}}.
	\end{align*}
	Applying uniform contraction on this interval with $\rho_{I_0}$
	gives us the final result.
\end{proof}

\begin{cor}
	\label{cor:one-dimensional-gradient-contraction}
	Let $g:\R\to\R$ be continuously differentiable and strictly increasing,
	with unique root at $x^*$. Assume $g'(x^*)>0$ and that
	\begin{align*}
		\phi(s):=(g^{-1})'(s)	
	\end{align*}
	is convex on $\mathrm{Range}(g)$. 
	For $0<\eta<2/g'(x^*)$, let $F(x) \coloneqq x - \eta g(x)$. Then
	$\abs{g(F(x))} < \abs{g(x)}$ holds for all $x\neq x^*$.
	Moreover, for every initialization $x_0$,
	there exists some $\rho < 1$ depending on $x_0$ such that 
	\begin{align*}
		\abs{g(x_t)} \leq \rho^t \abs{g(x_0)}.
	\end{align*}
\end{cor}
\begin{proof}
	For every $q>0$ such that $q$ and $-q$ belong to $\mathrm{Range}(g)$, 
	Jensen's inequality gives
	\begin{align*}
		g\inverse(q)-g\inverse(-q)
		= \int_{-q}^q \phi(s)\,d s
		\geq 2q\phi(0)
		= \frac{2q}{g'(x^*)}.
	\end{align*}
	Thus \labelcref{eq:opposite-grad-separation} holds with
	$\Lambda=g'(x^*)$, and the result follows from
	\cref{prop:one-dimensional-gradient-contraction}.
\end{proof}

%% file: refs.bib
@inproceedings{wu2023implicit,
  author       = {Jingfeng Wu and
                  Vladimir Braverman and
                  Jason D. Lee},
  title        = {{Implicit Bias of Gradient Descent for Logistic Regression at the Edge
                  of Stability}},
  booktitle    = {Advances in Neural Information Processing Systems 36, {NeurIPS}},
  year         = {2023},
}

@book{sternberg2010dynamical,
  title={{Dynamical Systems}},
  author={Sternberg, Shlomo},
  year={2010},
  publisher={Dover Publications}
}

@article{meng2024gradient,
  title={Gradient Descent on Logistic Regression with Non-Separable Data and Large Step Sizes},
  author={Meng, Si Yi and Orvieto, Antonio and Cao, Daniel Yiming and De Sa, Christopher},
  journal={arXiv:2406.05033},
  year={2024}
}

@book{nesterov2018lectures,
  title={{Lectures on Convex Optimization}},
  author={Nesterov, Yurii},
  volume={137},
  year={2018},
  publisher={Springer}
}

@article{soudry2018implicit,
  author       = {Daniel Soudry and
                  Elad Hoffer and
                  Mor Shpigel Nacson and
                  Suriya Gunasekar and
                  Nathan Srebro},
  title        = {{The Implicit Bias of Gradient Descent on Separable Data}},
  journal      = {Journal of Machine Learning Research},
  volume       = {19},
  pages        = {70:1--70:57},
  year         = {2018},
}

@inproceedings{wu2024large,
  author       = {Jingfeng Wu and
                  Peter L. Bartlett and
                  Matus Telgarsky and
                  Bin Yu},
  title        = {{Large Stepsize Gradient Descent for Logistic Loss: Non-Monotonicity
                  of the Loss Improves Optimization Efficiency}},
  booktitle    = {The Thirty Seventh Annual Conference on Learning Theory, {COLT}},
  series       = {Proceedings of Machine Learning Research},
  volume       = {247},
  pages        = {5019--5073},
  publisher    = {{PMLR}},
  year         = {2024},
}

@article{bach2010self,
  author       = {Francis R. Bach},
  title        = {Self-concordant analysis for logistic regression},
  journal      = {CoRR},
  volume       = {abs/0910.4627},
  year         = {2009},
}

@article{salton1988term,
  title={Term-weighting approaches in automatic text retrieval},
  author={Salton, Gerard and Buckley, Christopher},
  journal={Information processing \& management},
  volume={24},
  number={5},
  pages={513--523},
  year={1988},
  publisher={Elsevier}
}

@book{manning2008introduction,
  author       = {Christopher D. Manning and
                  Prabhakar Raghavan and
                  Hinrich Sch{\"{u}}tze},
  title        = {Introduction to information retrieval},
  publisher    = {Cambridge University Press},
  year         = {2008},
  isbn         = {978-0-521-86571-5},
}

@inproceedings{wang2017normface,
  author       = {Feng Wang and
                  Xiang Xiang and
                  Jian Cheng and
                  Alan Loddon Yuille},
  title        = {{NormFace: L\({}_{\mbox{2}}\) Hypersphere Embedding for Face Verification}},
  booktitle    = {Proceedings of the 2017 {ACM} on Multimedia Conference, {MM}},
  pages        = {1041--1049},
  publisher    = {{ACM}},
  year         = {2017},
}

@inproceedings{deng2019arcface,
  author       = {Jiankang Deng and
                  Jia Guo and
                  Niannan Xue and
                  Stefanos Zafeiriou},
  title        = {{ArcFace: Additive Angular Margin Loss for Deep Face Recognition}},
  booktitle    = {{IEEE} Conference on Computer Vision and Pattern Recognition, {CVPR}},
  pages        = {4690--4699},
  publisher    = {Computer Vision Foundation / {IEEE}},
  year         = {2019},
}

@inproceedings{chen2020simple,
  author = {Chen, Ting and Kornblith, Simon and Norouzi, Mohammad and Hinton, Geoffrey},
  title = {A simple framework for contrastive learning of visual representations},
  year = {2020},
  publisher = {JMLR.org},
  booktitle = {Proceedings of the 37th International Conference on Machine Learning},
  articleno = {149},
  numpages = {11},
  series = {ICML'20}
}

@article{lessard2016analysis,
  author       = {Laurent Lessard and
                  Benjamin Recht and
                  Andrew K. Packard},
  title        = {{Analysis and Design of Optimization Algorithms via Integral Quadratic Constraints}},
  journal      = {{SIAM Journal on Optimization}},
  volume       = {26},
  number       = {1},
  pages        = {57--95},
  year         = {2016},
}

@article{de2017worst,
  author       = {Etienne de Klerk and
                  Fran{\c{c}}ois Glineur and
                  Adrien B. Taylor},
  title        = {On the worst-case complexity of the gradient method with exact line
                  search for smooth strongly convex functions},
  journal      = {{Optimization Letters}},
  volume       = {11},
  number       = {7},
  pages        = {1185--1199},
  year         = {2017},
}

@inproceedings{wu2026large,
  author       = {Jingfeng Wu and
                  Pierre Marion and
                  Peter L. Bartlett},
  title        = {{Large Stepsizes Accelerate Gradient Descent for Regularized Logistic
                  Regression}},
  booktitle    = {Advances in Neural Information Processing Systems 38, {NeurIPS}},
  year         = {2025},
}

@inproceedings{crawshaw2026tight,
  author       = {Michael Crawshaw and
                  Mingrui Liu},
  title        = {{Tight Bounds for Logistic Regression with Large Stepsize Gradient
                  Descent in Low Dimension}},
  booktitle    = {The Thirty Ninth Annual Conference on Learning Theory, {COLT}},
  series       = {Proceedings of Machine Learning Research},
  volume       = {336},
  pages        = {1575--1610},
  publisher    = {{PMLR}},
  year         = {2026},
}

@article{chen2026finite,
  title={{Finite-Sample Performance of Gradient Descent in Logistic Regression with Gaussian Design}},
  author={Chen, Junren and Mazumdar, Arya},
  journal={arXiv:2606.21683},
  year={2026}
}

@inproceedings{liang2026gradient,
title={{Gradient Descent with Large Step Sizes: Chaos and Fractal Convergence Region}},
author={Shuang Liang and Guido Montufar},
booktitle={The Fourteenth International Conference on Learning Representations, {ICLR}},
year={2026},
}

@inproceedings{tyurin2025logistic,
  author       = {Alexander Tyurin},
  title        = {{From Logistic Regression to the Perceptron Algorithm: Exploring Gradient
                  Descent with Large Step Sizes}},
  booktitle    = {Thirty-Ninth {AAAI} Conference on Artificial Intelligence},
  pages        = {20938--20946},
  publisher    = {{AAAI} Press},
  year         = {2025},
}

@inproceedings{bao2026any,
  author       = {Han Bao and
                  Shinsaku Sakaue and
                  Yuki Takezawa},
  title        = {{Any-stepsize Gradient Descent for Separable Data under Fenchel-Young
                  Losses}},
  booktitle    = {Advances in Neural Information Processing Systems 38, {NeurIPS}},
  year         = {2025},
}

@article{zhang2026minimax,
  author  = {Ruiqi Zhang and Jingfeng Wu and Licong Lin and Peter L. Bartlett},
  title   = {{Minimax Optimal Convergence of Gradient Descent in Logistic Regression via Large and Adaptive Stepsizes}},
  journal = {Journal of Machine Learning Research},
  year    = {2026},
  volume  = {27},
  number  = {124},
  pages   = {1--31},
}

@inproceedings{axiotis2023gradient,
  author       = {Kyriakos Axiotis and
                  Maxim Sviridenko},
  title        = {{Gradient Descent Converges Linearly for Logistic Regression on Separable
                  Data}},
  booktitle    = {International Conference on Machine Learning, {ICML}},
  series       = {Proceedings of Machine Learning Research},
  volume       = {202},
  pages        = {1302--1319},
  publisher    = {{PMLR}},
  year         = {2023},
}

@article{kale2026stretched,
title={{Stretched Exponential Convergence of (Stochastic) Gradient Descent for Separable Logistic Regression}},
author={Sacchit Kale and Piyushi Manupriya and Pierre Marion and Francis Bach and Anant Raj},
journal={{Transactions on Machine Learning Research}},
issn={2835-8856},
year={2026},
}

@article{jury1964stability,
  title={On the stability of a certain class of nonlinear sampled-data systems},
  author={Jury, EL and Lee, B},
  journal={{IEEE Transactions on Automatic Control}},
  volume={9},
  number={1},
  pages={51--61},
  year={1964},
  publisher={IEEE}
}

@article{gonzaga2012stability,
  author={Gonzaga, Carlos A. C. and Jungers, Marc and Daafouz, Jamal},
  title        = {{Stability analysis of discrete-time Lur'e systems}},
  journal      = {Automatica},
  volume       = {48},
  number       = {9},
  pages        = {2277--2283},
  year         = {2012},
}

@article{jungers2024discrete,
  title={Discrete-time convergent nonlinear systems},
  author={Jungers, Marc and Shakib, Mohammad Fahim and van De Wouw, Nathan},
  journal={{IEEE Transactions on Automatic Control}},
  volume={69},
  number={10},
  pages={6731--6745},
  year={2024},
  publisher={IEEE}
}

@article{park2015improved,
  author       = {Bum Yong Park and
                  PooGyeon Park and
                  Nam Kyu Kwon},
  title        = {{An improved stability criterion for discrete-time Lur'e systems with
                  sector- and slope-restrictions}},
  journal      = {Automatica},
  volume       = {51},
  pages        = {255--258},
  year         = {2015},
}

@article{seiler2020construction,
  title={{Construction of periodic counterexamples to the discrete-time Kalman conjecture}},
  author={Seiler, Peter and Carrasco, Joaquin},
  journal={{IEEE Control Systems Letters}},
  volume={5},
  number={4},
  pages={1291--1296},
  year={2020},
  publisher={IEEE}
}

@article{heath2015second,
  author       = {William Paul Heath and
                  Joaqu{\'{\i}}n Carrasco and
                  Manuel de la Sen},
  title        = {{Second-order counterexamples to the discrete-time Kalman conjecture}},
  journal      = {Automatica},
  volume       = {60},
  pages        = {140--144},
  year         = {2015},
}

@article{kalman1957physical,
  title={Physical and mathematical mechanisms of instability in nonlinear automatic control systems},
  author={Kalman, Rudolf E.},
  journal={{Transactions of the American Society of Mechanical Engineers}},
  volume={79},
  number={3},
  pages={553--563},
  year={1957},
  publisher={American Society of Mechanical Engineers}
}

@article{lur1944theory,
  title={On the theory of stability of control systems},
  author={Lur'e, Anatoliy Isakovich and Postnikov, Vladimir N.},
  journal={Applied mathematics and mechanics},
  volume={8},
  number={3},
  pages={246--248},
  year={1944}
}

@article{kim2015stability,
  title={{Stability Analysis of Discrete-time Lur'e Systems with Slope-restricted Odd Monotonic Nonlinearities}},
  author={Kim, Kwang-Ki K and Braatz, Richard D},
  journal={arXiv:1509.01302},
  year={2015}
}

@article{tsypkin1964frequency,
  title={Frequency criteria for the absolute stability of nonlinear sampled-data systems},
  author={Tsypkin, Ya. Z.},
  journal={Automatika i Telemekhanika},
  volume={25},
  number={3},
  pages={281--289},
  year={1964}
}

@inproceedings{cohen2021gradient,
  author       = {Jeremy Cohen and
                  Simran Kaur and
                  Yuanzhi Li and
                  J. Zico Kolter and
                  Ameet Talwalkar},
  title        = {{Gradient Descent on Neural Networks Typically Occurs at the Edge of
                  Stability}},
  booktitle    = {9th International Conference on Learning Representations, {ICLR}},
  year         = {2021},
}

@inproceedings{chen2023beyond,
  author       = {Lei Chen and
                  Joan Bruna},
  title        = {{Beyond the Edge of Stability via Two-step Gradient Updates}},
  booktitle    = {International Conference on Machine Learning, {ICML}},
  series       = {Proceedings of Machine Learning Research},
  volume       = {202},
  pages        = {4330--4391},
  publisher    = {{PMLR}},
  year         = {2023},
}

@article{chen2024from,
  author       = {Xuxing Chen and
                  Krishna Balasubramanian and
                  Promit Ghosal and
                  Bhavya Agrawalla},
  title        = {{From Stability to Chaos: Analyzing Gradient Descent Dynamics in Quadratic
                  Regression}},
  issn         = {2835-8856},
  journal      = {Transactions on Machine Learning Research},
  year         = {2024},
}

@article{liu2023implicit,
  title={{Implicit Bias of Deep Learning in the Large Learning Rate Phase: A Data Separability Perspective}},
  author={Liu, Chunrui and Huang, Wei and Xu, Richard Yi Da},
  journal={Applied Sciences},
  volume={13},
  number={6},
  pages={3961},
  year={2023},
  publisher={MDPI}
}

@inproceedings{song2023trajectory,
  author       = {Minhak Song and
                  Chulhee Yun},
  title        = {{Trajectory Alignment: Understanding the Edge of Stability Phenomenon
                  via Bifurcation Theory}},
  booktitle    = {Advances in Neural Information Processing Systems 36, {NeurIPS}},
  year         = {2023},
}

@article{hofmann2026map,
  title={{The Map Behind the Flow: Finite-Step Gradient Descent as a Dynamical System}},
  author={Hofmann, Thomas},
  journal={arXiv:2607.04993},
  year={2026}
}

@inproceedings{damian2023self,
  author       = {Alex Damian and
                  Eshaan Nichani and
                  Jason D. Lee},
  title        = {{Self-Stabilization: The Implicit Bias of Gradient Descent at the Edge
                  of Stability}},
  booktitle    = {The Eleventh International Conference on Learning Representations,
                  {ICLR}},
  year         = {2023},
}

@article{meurer2017sympy,
  title={{SymPy: symbolic computing in Python}},
  author={Meurer, Aaron and Smith, Christopher P. and Paprocki, Mateusz and {\v{C}}ert{\'\i}k, Ond{\v{r}}ej and Kirpichev, Sergey B. and Rocklin, Matthew and Kumar, AMiT and Ivanov, Sergiu and Moore, Jason K. and Singh, Sartaj and others},
  journal={PeerJ Computer Science},
  volume={3},
  pages={e103},
  year={2017},
  publisher={PeerJ Inc.}
}
